\documentclass[letterpaper]{article} 
\usepackage{aaai25}  
\usepackage{times}  
\usepackage{helvet}  
\usepackage{courier}  
\usepackage[hyphens]{url}  
\usepackage{graphicx} 
\usepackage{natbib}  
\usepackage{caption} 
\usepackage{algorithm}
\usepackage{algorithmic}

\usepackage{newfloat}
\usepackage{listings}
\DeclareCaptionStyle{ruled}{labelfont=normalfont,labelsep=colon,strut=off} 
\floatstyle{ruled}
\newfloat{listing}{tb}{lst}{}
\floatname{listing}{Listing}
\usepackage{booktabs}
\usepackage{amsmath}
\usepackage{amssymb}
\usepackage{mathtools}
\usepackage{amsthm}

\usepackage{url}

\usepackage{natbib}
\usepackage{dsfont}

\usepackage{amsfonts}
\usepackage[mathscr]{eucal}
\usepackage{nicefrac}

\usepackage{tabularx}
\usepackage{multirow}
\usepackage{diagbox}

\usepackage[title]{appendix}

\DeclareMathOperator*{\argmin}{\arg\!\min}
\DeclareMathOperator*{\argmax}{\arg\!\max}

\newtheorem{theorem}{Theorem}
\newtheorem{corollary}[theorem]{Corollary}

\newcolumntype{Y}{>{\centering\arraybackslash}X}

\newcounter{appsection}  
\renewcommand{\theappsection}{\Alph{appsection}}  

\newcounter{appsubsection}[appsection]  
\renewcommand{\theappsubsection}{\theappsection.\arabic{appsubsection}}

\newcommand{\appsectiontitle}[1]{%
    \refstepcounter{appsection} 
    \noindent\textbf{\Large Appendix \theappsection. \; #1}  
    \vspace{0.5em}  
}

\newcommand{\appsubsectiontitle}[1]{%
    \refstepcounter{appsubsection} 
    \noindent\textbf{\large Appendix \theappsubsection. \; #1}  
    \vspace{0.5em}  
    \noindent  
}

\title{OneBatchPAM: A Fast and Frugal K-Medoids Algorithm}
\author{
    Antoine de Mathelin\textsuperscript{\rm 1},
    Nicolas Enrique Cecchi\textsuperscript{\rm 1},
    François Deheeger\textsuperscript{\rm 2},\\
    Mathilde Mougeot\textsuperscript{\rm 1},
    Nicolas Vayatis\textsuperscript{\rm 1}
}
\affiliations{
    \textsuperscript{\rm 1}Centre Borelli, Université Paris-Saclay, CNRS, ENS Paris-Saclay \\
    \textsuperscript{\rm 2}Michelin \\
    antoine.de\_mathelin@ens-paris-saclay.fr

}

\usepackage{bibentry}

\begin{document}

\maketitle

\begin{abstract}
    This paper proposes a novel $k$-medoids approximation algorithm to handle large-scale datasets with reasonable computational time and memory complexity. We develop a local-search algorithm that iteratively improves the medoid selection based on the estimation of the $k$-medoids objective. A single batch of size $m \ll n$ provides the estimation, which reduces the required memory size and the number of pairwise dissimilarities computations to $\mathcal{O}(mn)$, instead of $\mathcal{O}(n^2)$ compared to most $k$-medoids baselines. We obtain theoretical results highlighting that a batch of size $m=\mathcal{O}(\log(n))$ is sufficient to guarantee, with strong probability, the same performance as the original local-search algorithm. Multiple experiments conducted on real datasets of various sizes and dimensions show that our algorithm provides similar performances as state-of-the-art methods such as FasterPAM and BanditPAM++ with a drastically reduced running time.
\end{abstract}

\begin{links}
\link{Code}{https://github.com/antoinedemathelin/obpam}
\end{links}

\section{Introduction}
\label{introduction}

The $k$-medoids problem consists in choosing $k$ medoids from a set of $n$ points $\mathcal{X}_n$, minimizing the sum of the pairwise dissimilarities between the $n$ points and their nearest medoid. This problem has many uses in machine learning, in particular for clustering, subset selection and active learning \citep{bhat2014kmedoidsClustering, wei2015submodularity, kaushal2019learning, deMathelin2021discrepancy}. The $k$-medoids problem is related to $k$-medians, $k$-means and facility location \cite{schubert2021fastPAM}. One specificity of $k$-medoids is to consider generic dissimilarities (non-necessarily metric). In machine learning applications, the dissimilarity function can involve heavy computational costs, especially when computed between complex data types such as images, texts, or time series.

The $k$-medoids problem is a discrete optimization problem known to be NP-hard \citep{Hakimi1979NPhard}, for which a wide variety of approximation algorithms have been developed. Many $k$-medoids approximations are greedy or local-search algorithms, which improve a medoid selection sequentially by either adding or removing a medoid or swapping one medoid with another data point \cite{dohan2015kmedoidsInPractice}. The main local-search approach considered by the operations research communities is called PAM (Partitioning Around Medoid) \cite{kaufman1987PAMold, kaufman1990PAM}. This algorithm starts from an initial choice of $k$ points (potentially greedily selected) and then performs a series of "swaps". The state-of-the-art PAM algorithms are the FastPAM variants \citep{schubert2021fastPAM, schubert2022fastPAMpython}.

A major drawback of these approximation algorithms is the computational burden encountered for large values of $n$. Indeed, the main algorithms require the computation and in-memory conservation of pairwise dissimilarities between the $n$ points, resulting in a complexity of $\mathcal{O}(n^2)$. Nowadays, with the rise of Big Data, and the focus on reducing computational resources, there is a strong incentive to build algorithms that overcome this $\mathcal{O}(n^2)$ limitation.

Subsampling is a straightforward solution to reduce the number of dissimilarity calculations. The idea is to use an approximation algorithm (like PAM) on a subsample of size $m \ll n$ selected among the $n$ data points, resulting in a reduction of the time and memory complexities from $\mathcal{O}(n^2)$ to $\mathcal{O}(m^2)$. Previous works have proven that this simple approach yields appealing statistical guarantees over the approximation error for relatively small batch size $m$ \cite{mishra2001sublinear, thorup2005quickKmedian, mettu2004optimalTimeBounds, meyerson2004KmedianIndependantDataSize, huang2023powerOfUniform, guha2016clusteringDataStream, czumaj2007sublinear}. In this category of methods, the CLARA algorithm (Clustering LARge Applications) \cite{kaufman1986CLARAold, kaufman2008CLARA} is the most commonly used. The main drawback of the subsampling approach is the loose approximation of considering only the medoid candidates in the $m$ subsampled data points, resulting in worse clustering quality \cite{tiwari2020banditpam}. A recent method, BanditPAM, leverages Bandit algorithms to deal with this limitation \citep{tiwari2020banditpam, tiwari2023banditpam++}. BanditPAM keeps the $n$ data points as potential medoid candidates but only computes the dissimilarities for data points with high medoid potential, thus reducing the number of pairwise dissimilarity computations to $\mathcal{O}(n \log(n))$ for one medoid selection or one swap step of the PAM algorithm. Although BanditPAM provides a medoid selection close to PAM (in terms of $k$-medoids objective), the Bandit-based framework requires the computation of new pairwise dissimilarities at each medoid selection, which then results in computing $\mathcal{O}(T n \log(n))$ pairwise dissimilarities, with $T$ the number of iterations of the algorithm.

In this paper, we propose an alternative approach to address the $\mathcal{O}(n^2)$ limitations of local-search $k$-medoids algorithms. To avoid computing new pairwise dissimilarities at each swap, we only compute the dissimilarities between the $n$ data points and a single batch of size $m$. Our theoretical analysis shows that $m = \mathcal{O}(\log(n))$ is sufficient to guarantee similar performances as FasterPAM with strong probability. Our algorithm called OneBatchPAM provides a $\mathcal{O}(T)$ speedup of time complexity compared to BanditPAM and a $\mathcal{O}(n / \log(n))$ speedup compared to FasterPAM for similar performances. We show through several experiments, conducted on real datasets, that OneBatchPAM proposes an efficient time / objective trade-off compared to multiple $k$-medoids algorithms.

\section{Related Works}

\subsection{Approximation algorithms for k-medoids}

The $k$-medoids problem is related to facility locations, $k$-medians (or $p$-medians) and $k$-means problems. A detailed comparison of these problems is given in \cite{schubert2021fastPAM}. In a nutshell, the main $k$-medoids particularity is to consider generic dissimilarities (non-necessarily metric as in $k$-medians) and to constrain the $k$ medoids to belong to the dataset $\mathcal{X}_n$ (unlike $k$-means). $k$-medoids can then be seen as a special case of the facility location problem, where at most $k$ facilities, belonging to the set of clients, can be opened with cost zero. The metric $k$-medoids problem is often considered, in which case the problem is similar to $k$-medians over discrete metric space \cite{schubert2021fastPAM}.

As solving the $k$-medoids problem is NP-hard, many algorithms have been developed to provide approximations in polynomial running time\footnote{i.e., a constant polynomial degree independent of $k$} \cite{kaufman1990PAM, charikar1999constantApproxKmedoids, li2013approximatingKmedian, bhat2014kmedoidsClustering}. A ``naive'' greedy approach selects the medoids sequentially by solving a $1$-medoid problem at each iteration. This approach is simple to implement and yields relatively good results in practice, but its theoretical approximation error in $\mathcal{O}(n)$ is quite large \cite{dohan2015kmedoidsInPractice}. It can be improved to $\mathcal{O}(\log(n))$ by the reverse greedy approach that starts with $n$ medoids and removes them one by one until reaching $k$ medoids \cite{chrobak2006reverseGreedy}. The most notable improvement to the greedy approach is the PAM algorithm \cite{kaufman1987PAMold, kaufman1990PAM}. It greedily initializes the set of medoids and then performs a series of swaps from one medoid to one non-medoid that improve the total objective. In the metric case, this local search approach provides a constant approximation ratio of $5$ which can be reduced to $3+\epsilon$ when swapping multiple medoids at each iteration \cite{arya2001local5approx}. Assuming pairwise dissimilarities are precomputed, the time complexity of the seminal PAM algorithm is $\mathcal{O}(T k n^2)$, with $T$ the number of swap steps. A notable recent improvement, called FastPAM \cite{schubert2021fastPAM, schubert2022fastPAMpython}, reduces the PAM's time complexity to $\mathcal{O}(T n^2)$ by using a smart decomposition of the swap evaluation \cite{schubert2021fastPAM}. We emphasize that the PAM algorithm and its variants are perhaps the most widespread approximation algorithms for $k$-medoids. It provides an appealing trade-off between approximation error and time complexity. Although its theoretical approximation ratio is $5$, the error is often much smaller in practical use-cases (less than $2\%$ \cite{schubert2021fastPAM}).

When the dissimilarity evaluation is costly, and/or when the available memory is restricted. All aforementioned algorithms are limited by the $\mathcal{O}(n^2)$ pairwise dissimilarities computation cost and by the $\mathcal{O}(n^2)$ memory requirement to store the computed dissimilarities. Our work then focuses on reducing the time complexity of PAM while keeping similar performance. We therefore do not consider algorithms that propose improvement over the PAM performance at the price of additional computational efforts, such as \citep{li2013approximatingKmedian, byrka2017improvedApproxKmedian, ren2022globalKmedoidsBB}.

\subsection{Subsampling Methods}

Subsampling consists in performing a $k$-medoids algorithm on a subsample $\mathcal{X}_m$ of the original dataset $\mathcal{X}_n$, of size $m \ll n$. For instance, the CLARA algorithm \citep{kaufman2008CLARA} uses PAM on a subsample $\mathcal{X}_m$. It has been shown that a $k$-medoids algorithm with constant approximation can be derived, with great probability, using a random uniform subsample of size $m \simeq \mathcal{O}(k \log(n))$ \cite{mishra2001sublinear}. The required size of the subsample has been further reduced to $\mathcal{O}(k \log(k))$ with deeper analysis \cite{meyerson2004KmedianIndependantDataSize, czumaj2007sublinear}, which is independent of $n$. In this perspective, the CLARA algorithm proposes the heuristic $m = 40 + 2k$ for the subsample's size \citep{kaufman2008CLARA}. With such a setting, this subsampling method can drastically reduce the number of pairwise dissimilarity computations from $\mathcal{O}(n^2)$ to $\mathcal{O}(k^2)$. CLARA repeatedly computes a $k$-medoids approximation over multiple random subsamples drawn uniformly from $\mathcal{X}_n$ and selects the best set of $k$ medoids based on the evaluation over the whole dataset $\mathcal{X}_n$. Although only $\mathcal{O}(k^2)$ dissimilarity computations are needed to perform PAM over the subsample, one evaluation step requires to compute $nk$ dissimilarities, resulting in a $\mathcal{O}(T p nk)$ time complexity, with $T$ the number of subsamples. The cost of the evaluation step can be mitigated by evaluating the medoid set over another subsample from $\mathcal{X}_n$ \cite{meyerson2004KmedianIndependantDataSize}, in the same spirit as hold-out validation in machine learning.

The primary drawback of the subsampling approach is the approximation error, which is theoretically twice as large as that of performing the same $k$-medoids approximation on the full dataset of $n$ data points \cite{mishra2001sublinear, meyerson2004KmedianIndependantDataSize, czumaj2007sublinear}. In practice, this leads to a noticeable decline in performance.

A recent method BanditPAM \cite{tiwari2020banditpam, tiwari2023banditpam++} proposes an interesting idea based on Bandit evaluation of the swap and initialization steps in PAM. At each step, the best local improvement is estimated using multi-armed bandit techniques. BanditPAM therefore does not need to compute all pairwise dissimilarities but only the ones useful to find the best swap. The drawback of such an approach is to compute new dissimilarities at each step resulting in $\mathcal{O}((T + k) n \log(n))$ dissimilarity computations, with $T$ the number of swap evaluations. In this work, we propose to instead compute all pairwise distance between the $n$ data points and a batch of size $m = \mathcal{O}(\log(n))$. The same dissimilarities are used to evaluate all swap steps, resulting in $\mathcal{O}(n \log(n))$ dissimilarity computations.

\subsection{k-means++ as a proxy for k-medoids}

$k$-means++ is first designed as a seeding algorithm for $k$-means. It iteratively samples data points from $\mathcal{X}_n$ with a probability proportional to the distance raised to the power $p$ to the already sampled points for any $\ell_p$ distance. As the output of $k$-means++ is a set of cluster centers in $\mathcal{X}_n$, and since the objective of $k$-means is the same as $k$-medoids when considering the Euclidean distance, this algorithm can be used as a natural proxy for $k$-medoids. The $k$-means++ algorithm provides a $\mathcal{O}(\log(k))$-approximation for $k$-means and $k$-medians \cite{arthur2007kmeans++}, which is generally worse than PAM but it only require $\mathcal{O}(kn)$ pairwise dissimilarity computations instead of $\mathcal{O}(n^2)$.

Since the seminal work of \cite{arthur2007kmeans++}, two primary directions have been pursued to improve $k$-means++: enhancing the approximation error and reducing the time complexity. To improve the approximation, local-search algorithms are employed. These algorithms typically involve a random selection process similar to $k$-means++, followed by a swap if the new selection yields a better clustering outcome. For instance, single-swap local-search methods require $\mathcal{O}((Z+k) n)$ pairwise distance computations, with $Z$ the number of swap steps, and $\mathcal{O}(Zkn)$ additional operations \cite{lattanzi2019LocalSearchKmeans++}. Multiple-swap approaches can further refine the clustering but at a higher computational cost, involving  $\mathcal{O}((Z t + k) n)$ distance computations and $\mathcal{O}(Z n k^{2t-1})$ additional operations, with $t$ the number of simultaneous swaps \cite{beretta2024multiSwapKmeans++, huang2024linearTimeMultiSwapKmean++}. To accelerate the process, \cite{bachem2016kmc2} introduced kmc2, which speeds up $k$-means++ to $\mathcal{O}(L k^2)$ distance computations, with $L$ a method's specific parameter. Other methods leverage the specificity of Euclidean distance. For example, by projecting the data onto one dimension \cite{charikar2023simpleOneDimensionProjection}, or leveraging specific nearest neighbor structure \cite{cohen2020fastAccurateKmean++} \cite{pelleg1999acceleratingKmeansGeometric}.

\subsection{Coreset for k-medians}

According to \cite{feldman2020coresetsurvey}, a coreset is a data summarization technique that selects a subsample from a large dataset, preserving the information needed to perform specific tasks such as linear regression or clustering. Specifically, given a dataset $\mathcal{X}_n$, an objective function $\mathcal{L}$, and a set of queries $\mathcal{Q}$, a coreset $\mathcal{X}_m$ is a subsample of $\mathcal{X}_n$ for which any query $q \in \mathcal{Q}$ yields a similar objective value when computed on the coreset as when computed on the entire dataset, i.e., $\mathcal{L}(q, \mathcal{X}_n) \simeq \mathcal{L}(q, \mathcal{X}_m)$. In the context of $k$-medians clustering, the clustering cost of any $k$ centers computed on a coreset is approximately the same as the cost of these centers on the entire dataset.

While there is no consensual definition of a coreset, it generally refers to a ``strong coreset'' where the objective computed on the coreset is a $(1+\epsilon)$-approximation of the objective computed over the entire dataset for any query (e.g., any set of $k$ centers). Coreset is sometimes equated with subsampling when the set of queries is restricted to the coreset itself \cite{huang2023powerOfUniform, harpeled2004coresetsKmedian}. When the set of queries includes any combination of $k$ centers from the entire dataset, coresets are similar in spirit to OnebatchPAM. However, the coreset literature focuses on constructing sets that provide a $(1+\epsilon)$-approximation for any query, whereas OnebatchPAM focuses on achieving results comparable to PAM.

Various coresets for $k$-medians have been proposed, aiming to find the minimal size that guarantees the $(1+\epsilon)$-approximation for any set of $k$ centers \cite{harpeled2004coresetsKmedian, chen2009coresetsKmedian}. The best-known result for discrete metric $k$-medians (similar to metric $k$-medoids) is provided by \cite{feldman2020coresetsurvey}, with $m = \mathcal{O}(k \log(n) \epsilon^{-2})$. However, constructing such a coreset has a running time of $\mathcal{O}(pnk)$, with $p$ the data dimension. This has been improved by \cite{cohen2021newCoresetFramework}, which provides a similar-sized coreset with a running time of $\mathcal{O}(nk)$. The coreset size $m = \mathcal{O}(k \log(n) \epsilon^{-2})$ has been shown to be the minimal size to guarantee the $(1+\epsilon)$-approximation \cite{cohen2022optimalLowerBoundCoreset}.

This size can be reduced when the constraints are relaxed, leading to what is known as a weak coreset \cite{feldman2011kmediancoresetupperbound}. A weak coreset guarantees the $(1+\epsilon)$-approximation for only a subset of queries \cite{feldman2020coresetsurvey, jaiswal2024universalweakcoreset}. Other definitions of coresets include those with additive and multiplicative error approximations, such as lightweight coresets \cite{bachem2018LightWeightCoreset}. Smaller coresets yielding the $(1+\epsilon)$-approximation guarantee can be constructed when considering Euclidean space \cite{cohen2021improvedCoreset, feldman2011kmediancoresetupperbound}, or constrained problems, such as capacitated clustering (uniform distribution between clusters) \cite{huang2023powerOfUniform, braverman2022powerofUniformSamplingCoreset} and fair constraint clustering \cite{schmidt2020fairCoreset}.

\section{From PAM to OneBatchPAM}

\subsection{Notations}

The four parameters $n, k, p, m \in \mathbb{N}^*$ respectively denote the number of data points, the number of medoids, the problem dimension and the batch size. We consider the space $\mathcal{X}$ with $\mathcal{X} \subset \mathbb{R}^p$ and $d : \mathcal{X} \times \mathcal{X} \to \mathbb{R}_+$ a measure of dissimilarity over $\mathcal{X}$. We consider the set $\mathcal{X}_n = \{x_1, ..., x_n \}$ of $n$ data points in $\mathcal{X}$. We denote $\mathcal{P}_k(\mathcal{X}_n)$ the set of all subsets of $\mathcal{X}_n$ of size $k$. We aim at solving the $k$-medoids selection problem:
\begin{equation}
    \min_{\mathcal{M} \in \mathcal{P}_k(\mathcal{X}_n)} \; \sum_{i=1}^n \, d(x_i, \mathcal{M}) ,
\end{equation}
with $d(x_i, \mathcal{M}) = \min_{\tilde{x} \in \mathcal{M}} d(x_i, \tilde{x})$. We denote by $\mathcal{L}$ the objective function, such that $\mathcal{L}(\mathcal{M}) = \frac{1}{n} \sum_{x \in \mathcal{X}_n} d(x, \mathcal{M})$ for any $\mathcal{M} \in \mathcal{P}_k(\mathcal{X}_n)$. In the following, we consider the common assumption that one dissimilarity computation requires $\mathcal{O}(p)$ time complexity \cite{tiwari2020banditpam, schubert2021fastPAM}.

\subsection{PAM, FastPAM and FasterPAM}

The local-search approximation algorithm called PAM, performs several ``swap'' steps that progressively improve the medoid selection. This is formally described by the following recurrence equation:
\begin{equation}
\label{swap-step}
    \mathcal{M}_{t+1} = \argmin_{\substack{x \in \mathcal{M}_t, \\ x' \in \mathcal{X}_n \setminus \mathcal{M}_t}} \sum_{i=1}^n \, d\left(x_i, \left(\mathcal{M}_t \setminus \{ x \} \right) \cup \{ x' \} \right) .
\end{equation}
For any $t \in \{0, ..., T-1\}$, with $T$ the number of iterations. In the original PAM algorithm, the initial medoid set $\mathcal{M}_0$ is built using a greedy algorithm \cite{kaufman1990PAM}. The swap step described in Equation (\ref{swap-step}) consists in removing one medoid from $\mathcal{M}_t$ and adding a non-medoid from $\mathcal{X}_n \setminus \mathcal{M}_t$. This algorithm theoretically provides a $5$-approximation \cite{arya2001local5approx}. In practical scenarios, however, the approximation error is often below $2 \%$ \cite{schubert2021fastPAM}.

As highlighted by Equation (\ref{swap-step}) the ``naive'' approach to perform a swap step requires computing the sum of dissimilarities for every swap pair $(x, x') \in \mathcal{M}_t \times \mathcal{X}_n \setminus \mathcal{M}_t$, which leads to $\mathcal{O}(kn^2)$ operations to perform one swap. The FastPAM algorithm introduced in \cite{schubert2021fastPAM} proposes a modification of PAM that yields a $\mathcal{O}(k)$ speed up. The main idea lies in the fact that for each $x_i$, only the removal of the nearest medoid will modify the value of $d(x_i, \mathcal{M}_t)$. Therefore, only one pass through $\mathcal{X}_n$ is needed to compute the impact of removing one medoid for all $k$ medoids. The complexity of one swap step then only requires $\mathcal{O}(n^2)$ operations. Moreover, \cite{schubert2021fastPAM} shows that random initializations of the medoids lead to similar results as the greedy initialization but save $\mathcal{O}(kn^2)$ operations. Finally, additional speedups are derived by eagerly swapping a medoid with a non-medoid as soon as an improvement is found. In theory, eager swapping still requires $\mathcal{O}(n^2)$ operations for one swap but, in practice, it significantly speeds up the algorithm. The FastPAM algorithm with these additional improvements is called FasterPAM.

As noticed by \cite{tiwari2020banditpam}, the main drawback of FastPAM and FasterPAM is that they require to compute every dissimilarity between each pair of data points in $\mathcal{X}_n$, with complexity $\mathcal{O}(p n^2)$. A solution proposed by \cite{schubert2021fastPAM} is FasterCLARA which uses FasterPAM on subsamples of $\mathcal{X}_n$. However, this solution comes with large approximation error in practice. To overcome this issue, we propose the OneBatchPAM algorithm.

\subsection{OneBatchPAM}
\label{onebatchpam-sec}

The OneBatchPAM idea is the following: for any $x$, it is not necessary to compute every distance to every $x_i$ to perform the exact same swaps as FasterPAM. An estimation of the objectives on a subsample is sufficient. Theorem \ref{theorem1} will show that only a subsample of size $m = \mathcal{O}(\log(n))$ is needed to find the same series of swaps as FasterPAM with great probability.

Formally, OneBatch involves choosing a subsample $\mathcal{X}_m = \{x_{\sigma(1)}, ..., x_{\sigma(m)} \}$ drawn from $\mathcal{X}_n$, with $\sigma: \! [\![ 1, m ]\!] \to [\![ 1, n ]\!]$ the mapping indice function. The subsample $\mathcal{X}_m$ is used to estimate the best swap to perform, such that:
\begin{equation}
\label{onebatch-swap-step}
    \mathcal{M}_{t+1} = \argmin_{\substack{x \in \mathcal{M}_t, \\ x' \in \mathcal{X}_n \setminus \mathcal{M}_t}} \sum_{j=1}^m \, d\left(x_{\sigma(j)}, \left(\mathcal{M}_t \setminus \{ x \} \right) \cup \{ x' \} \right) .
\end{equation}
Compared to Equation (\ref{swap-step}), the sum is now only computed over $\mathcal{X}_m$. This modification drastically reduces the time complexity while keeping similar performances as FasterPAM with high probability as proven in Theorem \ref{theorem1} and Corollary \ref{corollary1}.

It must be underlined that Equation (\ref{onebatch-swap-step}) is not equivalent to subsampling as the search space is still $\mathcal{X}_n$. In subsampling methods, such as CLARA, we would have $x' \in \mathcal{X}_m \setminus \mathcal{M}_t$ instead of $x' \in \mathcal{X}_n \setminus \mathcal{M}_t$. This difference has a significant impact on the approximation error. By reducing the search space to $\mathcal{X}_m \setminus \mathcal{M}_t$, subsampling methods multiply by two the theoretical approximation error and, in practice, degraded performances are indeed observed.

\begin{theorem}
\label{theorem1}
    Let $\mathcal{X}_m$ be a subsample uniformly drawn from $\mathcal{X}_n$. Let $D = \max_{(x, x') \in \mathcal{X}_n} d(x, x')$ and $\Delta$ be the smallest difference between two objectives computed by FasterPAM. Then, for any $\delta \in ]0, 1]$, the OneBatchPAM algorithm returns the same set of medoid as FasterPAM with probability at least $1-\delta$ if:
    \begin{equation}
    \label{theorem1-eq}
        m \geq \frac{4 D^2}{\Delta^2} \log\left(\frac{2 T n}{\delta}\right) .
    \end{equation}
    Where $\Delta = \underset{t \in [\![0, T]\!]}{\min} \underset{\substack{x \in \mathcal{M}_t, \\ x' \in \mathcal{X}_n \setminus \mathcal{M}_t}}{\min} | \mathcal{L}(\mathcal{M}_t) - \mathcal{L}(\mathcal{M}_t \setminus \{x\} \cup \{x'\})|$
\end{theorem}
\begin{proof}
    The proof follows the same framework as the proof of Theorem 1 in \cite{tiwari2020banditpam}. It consists in finding the minimal sample size which guarantees that the statistical error on the objectives remains smaller than the smallest objective difference, $\Delta$, with high probability. Consequently, OneBatchPAM performs the same swaps as FasterPAM. The detailed proof is reported in the supplementary materials.
\end{proof}
As stated by Theorem \ref{theorem1}, the dependence of $m$ with respect to $n$ is only $m = \mathcal{O}(\log(n))$. This implies a drastic reduction of the time complexity as formally described in the following corollary.
\begin{corollary}
\label{corollary1}
    The OneBatch PAM algorithm returns the same set of medoids as FasterPAM with arbitrarily high probability with time complexity:
    \begin{equation}
        \mathcal{O}\left( (p + T) n \log(n)  \right) .
    \end{equation}
\end{corollary}

Table 1 provides a detailed comparison of OneBatchPAM's complexity against other algorithms. OneBatchPAM achieves a complexity gain of $\mathcal{O}(n / \log(n))$ over FasterPAM due to subsampling, and at least a $\mathcal{O}(T)$ improvement over BanditPAM++, as it avoids computing new dissimilarities at each swap step. While OneBatchPAM may require more computational time compared to subsampling and $k$-means++, it offers a superior approximation error factor. It is important to note that the values in Table 1 are theoretical; in practical scenarios, the performance comparison between methods can vary. For example, the approximation errors for PAM-based algorithms are often significantly lower than $5$.

\begin{table}[ht]
\centering
\small
\begin{tabular}{l|c|c}
\toprule
Algorithm & Complexity & Approximation  \\
\midrule
FasterPAM  & $(p+T) n^2$ & $5$  \\
BanditPAM++ & $ p (T+k) n \log(n) $ & $5$ \\
\midrule
\textbf{OneBatchPAM} & $(p + T) n \log(n)$ & $5$ \\
\midrule
FasterCLARA  & $I \left((p + T) k^2 +  p k n \right)$ & $10$ \\
$k$-means++ & $p k n$ & $\log(k)$ \\
\bottomrule
\end{tabular}
\caption{Summary of theoretical time complexity and approximation error. $T$ is the number of swaps iterations and $I$ the number of subsamples.}
\label{complexity}
\end{table}

\textbf{How many iterations $T$ are needed?} Generally, the larger the value of $T$, the better the objective, but this also increases the time complexity. It is important to note that the algorithms may terminate before reaching $T$ swaps if a local minimum is attained. According to \cite{tiwari2023banditpam++} and \cite{schubert2021fastPAM}, in practice, the required number of swaps is typically $\mathcal{O}(k)$. If a threshold $\epsilon$ on the improvement is set instead of a maximum number of iterations, such that the algorithm terminates when no swap is $1-\epsilon$ better than the current medoid selection, then the number of swaps is at most $T = \mathcal{O}(\log(n) / \epsilon)$.

\textbf{How $\mathcal{X}_m$ should be sampled?} Theorem \ref{theorem1} demonstrates that uniform sampling is sufficient to obtain good guarantees with a relatively small subset. However, a natural question arises: can we improve this with a more specific selection method? One initial approach consists in modifying the dissimilarity between the subsampled points and themselves as follows: $d(x_{\sigma(j)}, x_{\sigma(j)}) = +\infty$ for any $j \in \{1, \ldots, m\}$. We empirically observed that this adjustment prevents the medoid selection from being biased toward the subsampled data points. A second approach is to reweight the uniform sample to correct any potential sample bias. Since all distances between $\mathcal{X}_n$ and $\mathcal{X}_m$ are computed to perform OneBatchPAM, we recommend using the nearest neighbor sample bias correction method from \cite{loog2012nearest}. In this method, the importance of the data point $x_{\sigma(j)}$ is proportional to the number of data points in $\mathcal{X}_n$ whose nearest neighbor in $\mathcal{X}_m$ is $x_{\sigma(j)}$. Additionally, specific sampling techniques, such as those used to build coresets, may also be considered \cite{bachem2018LightWeightCoreset}.

\section{Discussion and Limitations}

\textbf{Minimum sample size of OneBatchPAM derived in Theorem \ref{theorem1}}.
The factor $1 / \Delta$ in the sample size lower bound also appears in the theoretical time complexity of BanditPAM. It is implicitly assumed that the minimum objective difference, $\Delta$, is not null \cite{tiwari2020banditpam}. The inverse proportionality between $m$ and $\Delta^2$ indicates that OneBatchPAM may require a large subsample to perform the exact same swaps as FasterPAM if two objectives are close. This can happen if two data points $x, x' \in \mathcal{X}_n$ are close. In that case, OneBatchPAM may estimate that adding $x$ to the set of medoids instead of $x'$ is more efficient while FatserPAM may do the opposite. However, as the difference between both objectives is small, OneBatchPAM will likely return a set of medoids with close performance to the one of FasterPAM. This is confirmed in our empirical experiments where OneBatchPAM provides close objectives compared to FasterPAM (around $2 \%$ error) but not exactly the same. We emphasize that the purpose of Theorem \ref{theorem1} is essentially to highlight the dependence of $m$ relative to $n$. Indeed, many upper bound approximations are involved in the derivation of the Theorem's result, hence using the exact value of Equation (\ref{theorem1-eq}) for $m$ may be disproportionate. In practice, we do not estimate the ratio $D/\Delta$ to set the sample size, but instead choose a value proportional to $\log(n)$.

It is interesting to notice that the minimum sample size for OneBatchPAM does not directly depend on the number of medoids $k$. However, this dependence is somehow hidden in the number of swap steps $T$. As highlighted by \cite{schubert2021fastPAM}, when starting with a random medoid selection, one can expect at least $k$ swaps before reaching a local minimum.

\textbf{Comparison to BanditPAM and memory limitations of OneBatchPAM.} Both BanditPAM and OneBatchPAM rely on the estimation of the $k$-medoids objective to determine which swap to perform. However, they consider two different approaches for estimating this objective. BanditPAM gradually improves the objective's estimation of swap pair candidates using mini-batches while reducing the set of candidates as the estimation becomes more accurate. The process is repeated after each swap, as the update of the medoid set modifies the swap pairs' evaluation. This leads to a linear increase in pairwise dissimilarity computations relative to the number of iterations. In contrast, OneBatchPAM computes all pairwise dissimilarities between the entire dataset $\mathcal{X}_n$ and a subsample $\mathcal{X}_m$ only once, using these precomputed values for each swap step. Consequently, it avoids the linear scaling of dissimilarity computations with the number of iterations. It should be noted that this computational load reduction comes with an increase in memory consumption. Indeed, BanditPAM only requires $\mathcal{O}(n)$ memory space while OneBatchPAM needs $\mathcal{O}(n \log(n))$. Nevertheless, this memory usage is significantly more efficient than the $\mathcal{O}(n^2)$ memory requirement of FasterPAM.

\textbf{Comparison to coresets.} As discussed in the related works section, OneBatchPAM is closely associated with coresets used in the context of $k$-medians. The coresets literature essentially focuses on constructing subsets that provide a $(1+\epsilon)$-approximation for any $k$-medoids selection. This imposes a stronger constraint compared to OneBatchPAM, which focuses on achieving results similar to those of the PAM algorithm. This explains why the minimal sample size for OneBatchPAM $m = \mathcal{O}(\log(n))$ is smaller than the minimal size for coresets for $k$-medians clustering with discrete metric spaces, $m = \mathcal{O}(k \log(n) \epsilon^{-2})$ \cite{cohen2022optimalLowerBoundCoreset}. It is important to note, however, that the sample size for OneBatchPAM is derived from a uniform sample $\mathcal{X}_m$. Leveraging coreset construction techniques could potentially further reduce the required sample size and, consequently, the time complexity of OneBatchPAM.

\textbf{Overfitting for highly imbalanced datasets}. Overfitting is a potential risk for OneBatchPAM, especially when the batch is not representative of the full dataset. Overfitting issues especially arise in situations involving highly imbalanced datasets. For instance, if a small subset of points are very far from all others. In such a case, there is a low probability that any neighbors of these distant points will be included in the batch, potentially leaving these points ``not covered'' by any medoid at the end of the OneBatchPAM algorithm. A potential future improvement to our approach could be to construct the batch progressively, leveraging the computed distances to identify imbalances in the dataset and mitigate the issue by selecting data points that improve the ``representativeness'' of the batch.

\section{Experiments}

We conduct several experiments on real datasets to compare OneBatchPAM with state-of-the-art $k$-medoids algorithms in practical scenarios. Our implementation of OneBatchPAM is coded in Python with the Cython module. The experiments are run on a 8G RAM computer with 4 cores. The source code of the experiments is available on GitHub\footnote{\url{https://github.com/antoinedemathelin/obpam}}.

\subsection{Datasets and settings}

We conduct the experiments on the MNIST and CIFAR10 image datasets \citep{mnist-digits, krizhevsky2009cifar10} and $8$ UCI datasets \cite{Dua2019UCI}, arbitrarily selected, with various sizes and dimensions (cf. Table \ref{dataset-table}). The $\ell_1$ distance is used as the dissimilarity function. Experiments are performed for different values of $k$ in $\{10, 50, 100\}$. Each experiment is repeated $5$ times to compute the standard deviations.

We divide the datasets into two categories respectively called ``small scale'' and ``large scale'' to account for the fact that some algorithms cannot provide a medoid selection in reasonable time for datasets above $\sim 50000$ instances. In particular, FasterPAM is not able to handle the size of the MNIST dataset \cite{schubert2021fastPAM}.

\begin{table}[h!]
\small
\begin{tabularx}{\linewidth}{>{\hsize=1.5\hsize}X|>{\hsize=0.9\hsize}Y|>{\hsize=0.4\hsize}Y||>{\hsize=1.3\hsize}X|>{\hsize=1.1\hsize}Y|>{\hsize=0.7\hsize}Y}
\toprule
\multicolumn{3}{c||}{Small Scale} & \multicolumn{3}{c}{Large Scale} \\
  \midrule
  Dataset & $n$ & $p$ & Dataset & $n$ & $p$  \\
  \midrule
  abalone & 4,176 & 8 & CIFAR & 50,000 & 3072 \\
bankruptcy & 6,819 & 96 & MNIST & 60,000 & 784 \\
mapping & 10,545 & 28 & dota2 & 92,650 & 117 \\
drybean & 13,611 & 16 & gas & 416,153 & 9 \\
letter & 19,999 & 16 & covertype &  581,011 & 55 \\
\bottomrule
\end{tabularx}
\caption{Datasets Summary. $n$ and $p$ are respectively the dataset's size and dimension.}
\label{dataset-table}
\end{table}

\subsection{Competitors and Hyper-parameters}

The following two kinds of competitors are considered
\begin{itemize}
    \item \textbf{PAM Algorithms}. we consider the PAM variants: FasterPAM, BanditPAM++ and FasterCLARA, as well as the Alternate approach \cite{park2009Alternate} although it is not formally a PAM method. We use the official implementations of BanditPAM++\footnote{\url{https://github.com/motiwari/BanditPAM}} \citep{tiwari2023banditpam++}. The other algorithms are found in the Python library \texttt{kmedoids}\footnote{\url{https://github.com/kno10/python-kmedoids}}, providing the official implementation of FasterPAM \citep{schubert2022fastPAMpython}.
    \item \textbf{$k$-means++ Algorithms}. We consider the original $k$-means++ algorithms and the two variants introduced in the related works: kmc2 \cite{bachem2016kmc2} and $k$-means++ with local-search (LS-k-means++) \cite{lattanzi2019LocalSearchKmeans++}.
\end{itemize}

If nothing else is specified the default hyperparameters are selected for the method. For BanditPAM++, we consider the three different settings of swap iterations: $T \in \{0, 2, 5\}$. We noticed that larger values of this parameter lead to excessive running time. For FasterCLARA we consider two different settings for the number of subsampling repetitions: $I \in \{5, 50\}$. The sample size is set to $m = 80 + 4 k$ as suggested in \cite{schubert2021fastPAM}. Three different chain lengths are considered for kmc2: $L = \{20, 100, 200\}$ and two different number of local search iterations for LS-k-means++: $Z = \{5, 10\}$. When different values of a parameter $P$ are used for an algorithm $Alg$, we denote the corresponding variants by $Alg$-$P$.

For OneBatchPAM, we use a sample size of $m = 100 \log(k n)$. The four following subsampling techniques introduced in Section \ref{onebatchpam-sec} are considered: \textbf{Unif}: uniform sampling; \textbf{Debias}:  uniform sampling with $d(x_{\sigma(j)}, x_{\sigma(j)}) = +\infty$ for any $j \in \{1, \ldots, m\}$; \textbf{NNIW}: uniform sampling with nearest-neighbor importance weighting and \textbf{LWCS}: sample built through the ``lightweight coreset'' technique from \cite{bachem2018LightWeightCoreset}.

\begin{figure*}[ht]
    \centering
    \begin{minipage}{0.24\linewidth}
        \centering
        \includegraphics[width=\linewidth]{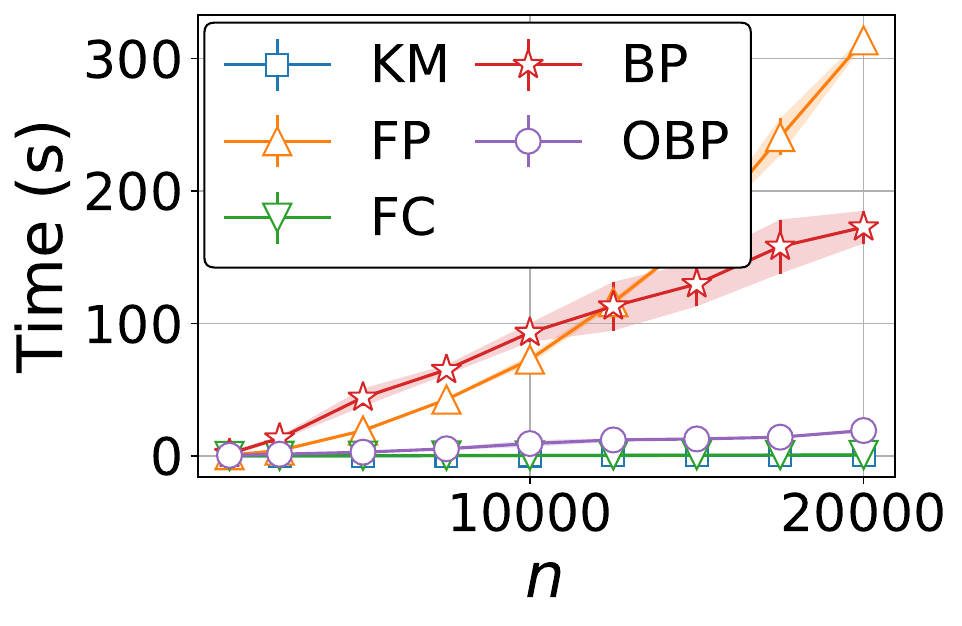} 
    \end{minipage}
    \begin{minipage}{0.24\linewidth}
        \centering
        \includegraphics[width=\linewidth]{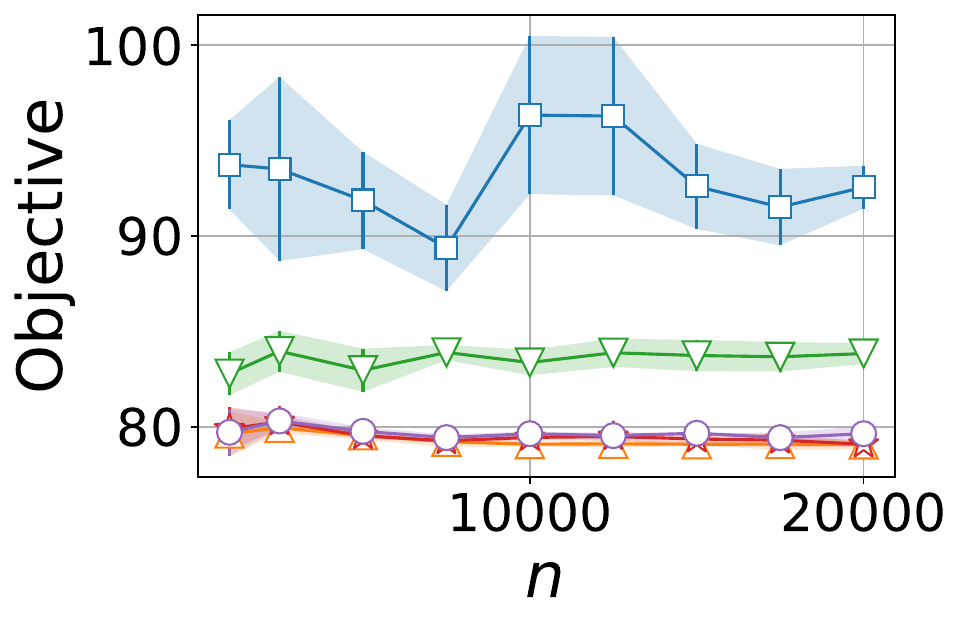} 
    \end{minipage} 
    \begin{minipage}{0.24\linewidth}
        \centering
        \includegraphics[width=\linewidth]{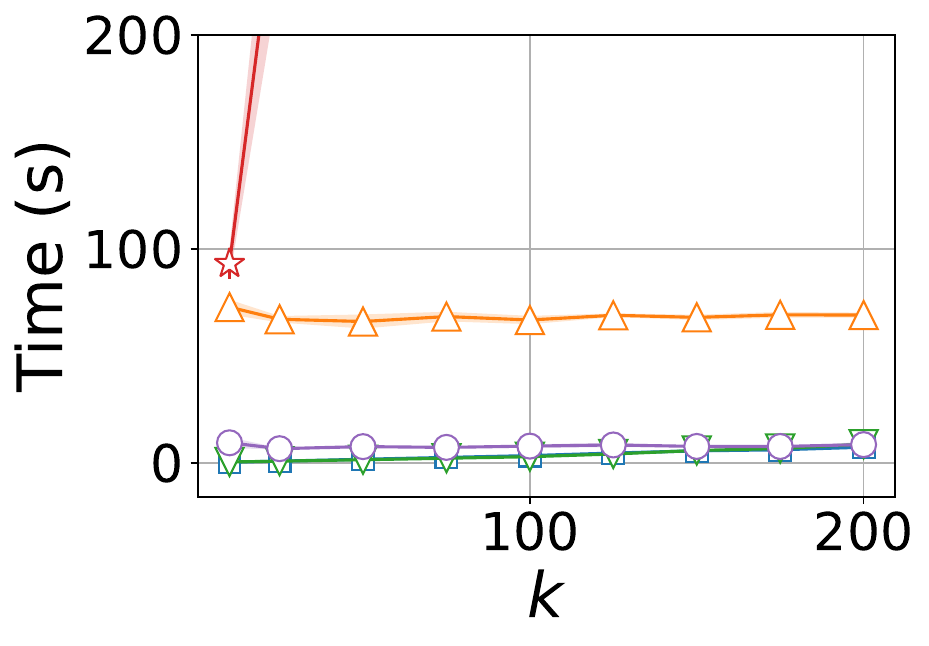}
    \end{minipage}
    \begin{minipage}{0.24\linewidth}
        \centering
        \includegraphics[width=\linewidth]{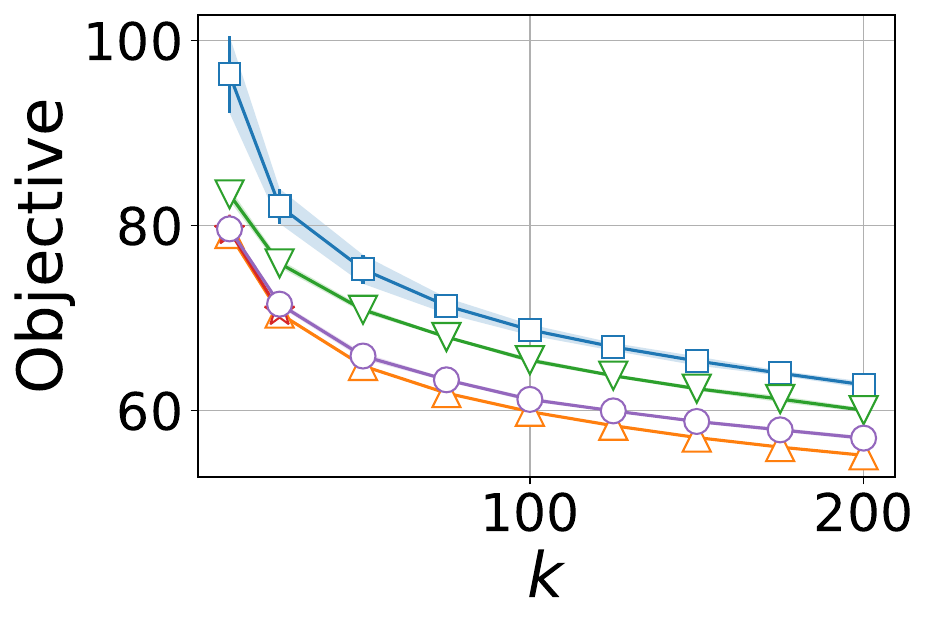} 
    \end{minipage}
    \caption{Evolution of the running time and objective on the MNIST dataset. Left: evolution as a function of $n$ for $k = 10$. Right: evolution as a function of $k$ for $n = 10000$. The results for five competitors are reported: k-means++ (KM), FasterPAM (FP), FasterCLARA-5 (FC), BanditPAM++-2 (BP), OneBatchPAM (OBP)}
    \label{evol-N}
\end{figure*}

\subsection{Results}

The methods are compared in terms of both objective value and computational time. To provide a normalized measure between datasets, we consider the ``delta relative objective'' ($\Delta$RO) and ``relative time'' (RT), defined for any algorithm $\mathcal{A}$ as follows:
\begin{equation}
    \Delta\text{RO}(\mathcal{A}) = \frac{\mathcal{L}(\mathcal{M}^\mathcal{A})}{\mathcal{L}(\mathcal{M}^{\mathcal{A}^*})} - 1 \; ; \; \text{RT}(\mathcal{A}) = \frac{T(\mathcal{A})}{T(\mathcal{A}^*)} .
\end{equation}
Where $\mathcal{M}^\mathcal{A}$ is the set of medoids selected by algorithm $\mathcal{A}$, $\mathcal{A}^*$ refers to the algorithm providing the best objective.

\smallskip

\smallskip

\textbf{Evolution of the objective and running time for different values of $n$ and $k$}. Figure \ref{evol-N} shows the objective and running time of five algorithms for different $(k, n)$ settings on the MNIST dataset. In each graph, OneBatchPAM ranks among the best methods both in terms of objective and running time. The time evolution of OneBatchPAM is similar to the one of $k$-means++ and FasterCLARA-5, and significantly smaller than BanditPAM++ and FasterPAM, especially for large values of $n$. Additionally, the objective evolution of OneBatchPAM closely matches that of FasterPAM, while FasterCLARA-5 and $k$-means++ provide larger objective values.

\textbf{Aggregated Results.} Table \ref{tab:small} presents the averaged results over the three values of $k \in \{10, 50, 100\}$, the five repetitions of the experiments and the respective five ``small scale'' and ``large scale'' datasets. The detailed results per dataset and value of $k$ are reported in the supplementary materials. As expected, FasterPAM provides the best objective and Random the fastest medoid selection for the small scale experiments. The OneBatchPAM variants reduce the computational burden of FasterPAM by a factor of $7$ on average (RT = $15\%$) for a small penalization of the objective value ($1.7 \%$ compared to FasterPAM for the NNIW variant). This observation highlights the efficiency of OneBatchPAM to provide a fast and accurate medoid selection. Notice that the time reduction factor increases with the number of samples. The relative time for OneBatchPAM is equal to $8.5 \%$ for the \textit{letter} dataset, which corresponds to a reduction factor of around $12$ (cf. detailed results in supplementary materials). We observe that $k$-means++ and FasterCLARA-$5$ are faster than OneBatchPAM (by a factor of around $7$ for FasterCLARA-$5$). However, the running time reduction comes with a significant penalization of the objective: respectively $13\%$ and $30\%$ for FasterCLARA-$5$ and $k$-means++. For large scale datasets, FasterPAM and BanditPAM++ fail to provide medoid selections within reasonable computational times, positioning OneBatchPAM as the method with the best objective ($\Delta$RO = 0). Similar to the small-scale experiments, FasterCLARA-$5$ is $7$ times faster than OneBatch but is $8\%$ worse in terms of objective. kmc2 is even faster, however, its objective is close to the random selection's objective.

Regarding the OneBatchPAM variants, we observe that debiasing offers a modest improvement compared to uniform sampling ($\sim 0.2\%$). The gain is higher for large values of $k$ (around $1\%$), as highlighted in the detailed results. While the LWCS method degrades performance (likely because LWCS is primarily designed to provide strong theoretical guarantees for $k$-means++ rather than PAM), the NNIW variant shows significant objective improvements (above $1.2\%$) over uniform sampling with comparable computational time. This observation supports the systematic use of nearest neighbor importance weighting in OneBatchPAM. Indeed, the pairwise dissimilarities needed to compute the importance weights are also required by the OneBatchPAM core algorithm, which explains why using NNIW has a negligible impact on the running time.

\begin{table}[ht]
    \centering
    \small
\begin{tabular}{l|cc|cc}
\toprule
\multirow{2}{*}{Method} & \multicolumn{2}{c|}{Small Scale} & \multicolumn{2}{c}{Large Scale} \\
  & RT & $\Delta$RO & RT & $\Delta$RO \\
\midrule
Random & \textbf{0.0} & 62.9 & \textbf{0.0} & 20.3 \\
FasterPAM & 100.0 & \textbf{0.0} & Na & Na \\
Alternate & 161.1 & 20.0 & Na & Na \\
FasterCLARA-5 & 2.8 & 13.0 & 15.0 & 8.0 \\
FasterCLARA-50 & 30.0 & 10.9 & 161.7 & 7.1 \\
\midrule
kmc2-20 & 14.5 & 31.3 & 0.5 & 18.2 \\
kmc2-100 & 72.2 & 31.9 & 2.4 & 17.6 \\
kmc2-200 & 153.6 & 33.0 & 5.2 & 18.6 \\
k-means++ & 1.6 & 30.4 & 78.8 & 18.4 \\
LS-k-means++-5 & 37.2 & 23.5 & 97.1 & 15.3 \\
LS-k-means++-10 & 73.1 & 20.1 & 121.6 & 13.7 \\
\midrule
BanditPAM++-0 & 930.2 & 3.6 & Na & Na \\
BanditPAM++-2 & 1670.1 & 2.8 & Na & Na \\
BanditPAM++-5 & 2880.7 & 2.2 & Na & Na \\
\midrule
OneBatchPAM-lwcs & 15.1 & 12.3 & 117.9 & 2.8 \\
OneBatchPAM-unif & 15.1 & 3.9 & 104.2 & 1.2 \\
OneBatchPAM-debias & 15.7 & 3.7 & 100.0 & 0.8 \\
OneBatchPAM-nniw & 15.5 & 1.7 & 100.0 & \textbf{0.0} \\
\bottomrule
\end{tabular}
\caption{Results Summary. The scores are averaged over the five repetitions of the experiment, the three values of $k \in \{10, 50, 100\}$ and the five respective ``small scale'' and ``large scale'' datasets. RT and $\Delta$RO are given in percentage. Standard deviations are reported in Appendix.}
    \label{tab:small}
\end{table}

\section{Conclusion and Perspectives}

This paper introduces OneBatchPAM, a novel $k$-medoids algorithm that accelerates FasterPAM by using a single batch of size $m = \mathcal{O}(\log(n))$ to estimate the objective. Our experiments demonstrate that OneBatchPAM is an efficient alternative to subsampling for handling large datasets within a reasonable running time while achieving performance similar to FasterPAM (with less than $2 \%$ error). Future work will focus on refining the subsampling process to further improve the running time and the accuracy of the medoid selection.

\clearpage

\onecolumn








\begin{appendices}

\setlength{\parindent}{0pt}

\appsectiontitle{Algorithm}

\vspace{1em}

This section presents the pseudo-code of the OneBatchPAM algorithm. For simplicity, we consider the ``Uniform'' variant where the sample $\mathcal{X}_m$ is selected uniformly at random in $\mathcal{X}_n$ without reweighting and the two variants Debias and NNIW. For Algorithm \ref{fasterpam}, we respectively define $\texttt{near}(j)$, $\texttt{sec}(j)$ as the indices in $[1, k]$ of the nearest and second nearest medoid to $x_{\sigma(j)}$ in $\mathcal{M}$. We denote $d_{\texttt{near}(j)}$, $d_{\texttt{sec}(j)}$ the corresponding dissimilarities between $x_{\sigma(j)}$ and its respective nearest and second nearest medoid in $\mathcal{M}$. The Approximated-FasterPAM algorithm is close to FasterPAM \cite{schubert2021fastPAM}. The difference lies in the loop of line 9, as the loop is performed only over the subsampled data points $x_{\sigma(j)}$.


\begin{algorithm}[H]
	\caption{OneBatchPAM}
	\label{algo1}
    	\begin{algorithmic}[1]
    	    \STATE \textbf{Inputs}: Data $\mathcal{X}_n$, number of medoids $k$, maximal number of iteration $T$, batch size $m$ 
            \STATE \textbf{Outputs}: Set of medoids $\mathcal{M}$
            \smallskip
            \STATE Uniformly select $\mathcal{X}_m \subset
            \mathcal{X}_n$ of size $m$ 
            \STATE Compute $d_{ij} = d(x_i, x_{\sigma(j)})$ for any $j \in [1, m]$ and any $i \in [1, n]$
            \STATE (For the NNIW variant) Compute $w_j$ according to \cite{loog2012nearest} and update $d_{ij} \leftarrow w_j d_{ij}$
            \STATE (For the Debias variant) Update $d_{jj} \leftarrow + \infty$
            \STATE Randomly select $\mathcal{M} \in \mathcal{P}_k(\mathcal{X}_n)$
            \STATE Approximated-FasterPAM$\left( \{ d_{ij} \}_{i \leq n, j 
 \leq m}, \mathcal{M}, T, k, n, m \right)$
\end{algorithmic}
\end{algorithm}

\begin{algorithm}[H]
\caption{Approximated-FasterPAM}
\label{fasterpam}
    \begin{algorithmic}[1]
        \STATE \textbf{Inputs}: $\{ d_{ij} \}_{i \leq n, j 
\leq m}$, $\mathcal{M}$, $k$, $T$, $n$, $m$ 
    \STATE \textbf{Outputs}: Set of medoids $\mathcal{M}$
    \smallskip
    \STATE For any $j \in [1, m]$ compute $\texttt{near}(j), \texttt{sec}(j), d_{\texttt{near}(j)}, d_{\texttt{sec}(j)}$
    \STATE For any $l \in [1, k]$, initialize $G_l = \sum_{j \in [1, m]} d_{\texttt{near}(j)} - d_{\texttt{sec}(j)}$
    \FOR {$1 \leq t \leq T$}
    \FOR {$1 \leq i \leq n$}
    \STATE Initialize $G_l^i \leftarrow G_l$ for any $l \in [1, k]$
    \STATE Initialize $G^i \leftarrow 0$
    \FOR {$1 \leq j \leq m$}
    \IF{$d_{ij} < d_{\texttt{near}(j)}$}
    \STATE $G^i \leftarrow G^i + d_{\texttt{near}(j)} - d_{ij}$
    \STATE $G_{\texttt{near}(j)}^i \leftarrow G_{\texttt{near}(j)}^i + d_{\texttt{sec}(j)} - d_{\texttt{near}(j)}$
    \ELSIF{$d_{ij} < d_{\texttt{sec}(j)}$}
    \STATE $G_{\texttt{near}(j)}^i \leftarrow G_{\texttt{near}(j)}^i + d_{\texttt{sec}(j)} - d_{\texttt{near}(j)}$
    \ENDIF
    \ENDFOR
    \STATE $l^* = \argmax_{l \in [1, k]} G_l^i$
    \STATE $G^i \leftarrow G^i + G_{l^*}^i$
    \IF{$G^i > 0$}
    \STATE Update $\mathcal{M}$: swap role of medoid of indice $l^*$ in $\mathcal{M}$ with $x_i$.
    \STATE Update $\texttt{near}(j), \texttt{sec}(j), d_{\texttt{near}(j)}, d_{\texttt{sec}(j)}$ and $G_l$ for any $l \in [1, k]$
    \ENDIF
    \ENDFOR
    \ENDFOR
\end{algorithmic}
\end{algorithm}

\clearpage

\appsectiontitle{Proof of Theorem \ref{theorem1}}

\setcounter{theorem}{0}
\begin{theorem}
    Let $\mathcal{X}_m$ be a subsample uniformly drawn from $\mathcal{X}_n$. Let $D = \max_{(x, x') \in \mathcal{X}_n} d(x, x')$ and $\Delta$ be the smallest difference between two objectives computed by FasterPAM. Then, for any $\delta \in ]0, 1]$, the OneBatchPAM algorithm returns the same set of medoid as FasterPAM with probability at least $1-\delta$ if:
    \begin{equation}
    \label{app-theorem1-eq}
        m \geq \frac{4 D^2}{\Delta^2} \log\left(\frac{2 T n}{\delta}\right) .
    \end{equation}
    Where $\Delta = \underset{t \in [0, T]}{\min} \underset{\substack{x \in \mathcal{M}_t, \\ x' \in \mathcal{X}_n \setminus \mathcal{M}_t}}{\min} | \mathcal{L}(\mathcal{M}_t) - \mathcal{L}(\mathcal{M}_t \setminus \{x\} \cup \{x'\})|$
\end{theorem}

\begin{proof}
For any $t \in [0, T]$, $\mathcal{M}_t \in \mathcal{P}_k(\mathcal{X}_n)$ denotes the medoid selection of FasterPAM after $t$ swaps.

We denote by $\widehat{\mathcal{L}}(\mathcal{M})$ the empirical risk for any $\mathcal{M} \in \mathcal{P}_k(\mathcal{X}_n)$ such that $\widehat{\mathcal{L}}(\mathcal{M}) = \frac{1}{m} \sum_{j=1}^m d(x_{\sigma(j)}, \mathcal{M})$.

At each swap step $t \in [0, T]$, the FasterPAM algorithm evaluates the objective of several pairs $(x, x')  \in \mathcal{M}_t \times \mathcal{X}_n / \mathcal{M}_t$. It compares it to the current objective $\mathcal{L}(\mathcal{M}_t)$ until finding a pair with a lower objective (if no such pair is found, the algorithm terminates). Let's denote $\mathcal{P}_t \subset \mathcal{M}_t \times \mathcal{X}_n$ the swap pairs evaluated by FasterPAM with a larger objective than $\mathcal{L}(\mathcal{M}_t)$ and $(x_t, x'_t) \in \mathcal{M}_t \times \mathcal{X}_n$ the swap pair selected by FasterPAM. Thus, for any $t \in [0, T]$ and any $(x, x') \in \mathcal{P}_t$ we have:
\begin{gather}
    \mathcal{L}(\mathcal{M}_t) <  \mathcal{L}(\mathcal{M}^{(x, x')}) \\
    \mathcal{L}(\mathcal{M}_t) > \mathcal{L}(\mathcal{M}^{(x_t, x_t')}) ,
\end{gather}
where $\mathcal{M}_t^{(x, x')} = \mathcal{M}_t \setminus \{x\} \cup \{x'\}$.

Let's consider $\delta \in ]0, 1]$, we define $\tilde{\delta} \in ]0, 1]$ as follows:
\begin{equation}
    \tilde{\delta} = \frac{\delta}{2 T n^2}
\end{equation}
Let's consider a subsample size $m$ verifying Equation (\ref{app-theorem1-eq}). It can be noticed that:
\begin{equation}
    m \geq \frac{2 D^2}{\Delta^2} \log\left( \frac{1}{\tilde{\delta}} \right)
\end{equation}

To prove that OneBatchPAM selects the same swap pairs as FasterPAM, we have to show that, for any $t \in [0, T]$, the objective estimation for any swap pairs in $\mathcal{P}_t$ is lower than the current objective estimation, while the objective estimation for the pair $(x_t, x'_t)$ is larger, i.e., for any $t \in [0, T]$ and any $(x, x') \in \mathcal{P}_t$
\begin{gather}
    \widehat{\mathcal{L}}(\mathcal{M}_t) <  \widehat{\mathcal{L}}(\mathcal{M}^{(x, x')}) \\
    \widehat{\mathcal{L}}(\mathcal{M}_t) > \widehat{\mathcal{L}}(\mathcal{M}^{(x_t, x_t')}) ,
\end{gather}
For this purpose, we will show that the probability of the events $\widehat{\mathcal{L}}(\mathcal{M}_t) \geq  \widehat{\mathcal{L}}(\mathcal{M}^{(x, x')})$ and $\widehat{\mathcal{L}}(\mathcal{M}_t) \leq \widehat{\mathcal{L}}(\mathcal{M}^{(x_t, x_t')})$ is upper-bounded by $\tilde{\delta}$.

Let $t \in [0, T]$ and $\mathcal{M}_t \in \mathcal{P}_k(\mathcal{X}_n)$, by the Hoeffding inequality, we have for any $(x, x') \in \mathcal{P}_t$:
\begin{gather}
\label{hoef1}
\mathbb{P}\left(\widehat{\mathcal{L}}(\mathcal{M}_t) - \mathcal{L}(\mathcal{M}_t) \geq D \sqrt{\frac{\log(1 / \tilde{\delta})}{2 m}} \right) \leq \tilde{\delta} \\
\label{hoef2}
\mathbb{P}\left(\mathcal{L}(\mathcal{M}_t^{(x, x')}) - \widehat{\mathcal{L}}(\mathcal{M}_t^{(x, x')}) \geq D \sqrt{\frac{\log(1 / \tilde{\delta})}{2 m}} \right) \leq \tilde{\delta} 
\end{gather}
And,
\begin{gather}
\label{hoef3}
\mathbb{P}\left(\mathcal{L}(\mathcal{M}_t) - \widehat{\mathcal{L}}(\mathcal{M}_t) \geq D \sqrt{\frac{\log(1 / \tilde{\delta})}{2 m}} \right) \leq \tilde{\delta} \\
\label{hoef4}
\mathbb{P}\left(\widehat{\mathcal{L}}(\mathcal{M}_t^{(x_t, x_t')}) - \mathcal{L}(\mathcal{M}_t^{(x_t, x_t')}) \geq D \sqrt{\frac{\log(1 / \tilde{\delta})}{2 m}} \right) \leq \tilde{\delta} 
\end{gather}

Let's consider $(x, x') \in \mathcal{P}_t$, to simplify the notations, we define the five quantities: $C = D \sqrt{\frac{\log(1 / \tilde{\delta})}{2 m}}$, $L = \mathcal{L}(\mathcal{M}_t)$, $\widehat{L} = \widehat{\mathcal{L}}(\mathcal{M}_t)$, $L_x = \mathcal{L}(\mathcal{M}_t^{(x, x')})$, $\widehat{L}_x = \widehat{\mathcal{L}}(\mathcal{M}_t^{(x, x')})$, 

We have:
\begin{equation}
\begin{split}
    \mathbb{P}\left( \widehat{L}_x \leq \widehat{L} \right) & =  \mathbb{P}\left( \left\{\widehat{L}_x \leq \widehat{L} \right\} \cap \left\{\widehat{L}_x > L_x - C \right\} \right) + \mathbb{P}\left( \left\{\widehat{L}_x \leq \widehat{L} \right\} \cap \left\{\widehat{L}_x \leq L_x - C \right\} \right) \\
    & \leq \mathbb{P}\left(  L_x - C < \widehat{L}  \right) + \mathbb{P}\left( \widehat{L}_x \leq L_x - C \right) \\
    & \leq \mathbb{P}\left(  L_x - C < \widehat{L}  \right) + \tilde{\delta} \\
    & \leq \mathbb{P}\left( \left\{ L_x - C < \widehat{L} \right\} \cap \left\{ \widehat{L} < L + C \right\} \right) + \mathbb{P}\left( \left\{ L_x - C < \widehat{L} \right\} \cap \left\{ \widehat{L} \geq L + C \right\} \right) + \tilde{\delta} \\
    & \leq \mathbb{P}\left(  L_x - C < L + C \right) + \mathbb{P}\left( \widehat{L} \geq L + C \right) + \tilde{\delta} \\
    & \leq \mathbb{P}\left(  L_x - L < 2 C \right) + 2 \tilde{\delta} ,
\end{split}
\end{equation}
by using the respective Equations (\ref{hoef2}) and (\ref{hoef1}) for the third and sixth lines.

Therefore,
\begin{equation}
    \mathbb{P}\left( \widehat{\mathcal{L}}(\mathcal{M}_t^{(x, x')}) < \widehat{\mathcal{L}}(\mathcal{M}_t) \right) \leq \mathbb{P}\left( \mathcal{L}(\mathcal{M}_t^{(x, x')}) - \mathcal{L}(\mathcal{M}_t) < 2 C \right) + 2 \tilde{\delta}
\end{equation}

The quantity $\mathcal{L}(\mathcal{M}_t^{(x, x')}) - \mathcal{L}(\mathcal{M}_t)$ is positive, as $(x, x')$ is not a swap pair. Then, by definition of $\Delta$, we have:
\begin{equation}
    \mathcal{L}(\mathcal{M}_t^{(x, x')}) - \mathcal{L}(\mathcal{M}_t) \geq \Delta
\end{equation}
Then:
\begin{equation}
    \mathbb{P}\left( \widehat{\mathcal{L}}(\mathcal{M}_t^{(x, x')}) \leq \widehat{\mathcal{L}}(\mathcal{M}_t) \right) \leq \mathbb{P}\left( \Delta < 2 C \right) + 2 \delta
\end{equation}

If $m$ verifies Equation (\ref{app-theorem1-eq}), we have:
\begin{equation}
    2 C  \leq 2 D \sqrt{\frac{\log(1 / \tilde{\delta})}{4 \frac{D^2}{\Delta^2} \log(1 / \tilde{\delta})}} \leq \Delta .
\end{equation}
Then, $\mathbb{P}\left( \Delta < 2 C \right) = 0$ and we conclude that:
\begin{equation}
    \mathbb{P}\left( \widehat{\mathcal{L}}(\mathcal{M}_t^{(x, x')}) \leq \widehat{\mathcal{L}}(\mathcal{M}_t) \right) \leq 2 \tilde{\delta}
\end{equation}

By using  Equations (\ref{hoef3}) and (\ref{hoef4}), a similar proof can be derived to show that:
\begin{equation}
    \mathbb{P}\left( \widehat{\mathcal{L}}(\mathcal{M}_t^{(x_t, x_t')}) \geq \widehat{\mathcal{L}}(\mathcal{M}_t) \right) \leq 2 \tilde{\delta} ,
\end{equation}

Let's denote $A$ the event: ``OneBatchPAM performs a different swap as FasterPAM'':
\begin{equation}
    A = \bigcup_{t \in [0, T-1]} \left( \bigcup_{(x, x') \in \mathcal{P}_t} \left\{ \widehat{\mathcal{L}}(\mathcal{M}_t^{(x, x')}) \leq \widehat{\mathcal{L}}(\mathcal{M}_t) \right\} \cup \left\{ \widehat{\mathcal{L}}(\mathcal{M}_t^{(x_t, x_t')}) \geq \widehat{\mathcal{L}}(\mathcal{M}_t) \right\} \right)
\end{equation}

Then,
\begin{equation}
\begin{split}
    \mathbb{P}\left( A \right) & \leq \sum_{t \in [0, T-1]} \left( \sum_{(x, x') \in \mathcal{P}_t} \mathbb{P}\left( \widehat{\mathcal{L}}(\mathcal{M}_t^{(x, x')}) \leq \widehat{\mathcal{L}}(\mathcal{M}_t) \right) +  \mathbb{P}\left( \widehat{\mathcal{L}}(\mathcal{M}_t^{(x_t, x_t')}) \geq \widehat{\mathcal{L}}(\mathcal{M}_t)\right) \right) \\
    & \leq 2 T n^2 \tilde{\delta} \\
    & \leq \delta
\end{split}
\end{equation}

Finally, it can be concluded that, if $m$ verifies Equation (\ref{app-theorem1-eq}), then OneBatchPAM performs the same swaps as FasterPAM (and thus returns the same set of medoids) with probability at least $1-\delta$.

\end{proof}

\clearpage

\appsectiontitle{Detailed Results}

\begin{table}[H]
\centering
\caption{\textbf{Results Summary}. The scores are averaged over the five repetitions of the experiment, the three values of $k \in [10, 50, 100]$ and the five respective ``small scale'' and ``large scale'' datasets. RT and $\Delta$RO are given in percentage. The standard deviations computed over the five repetitions of the experiments and averaged over the five datasets and the three values of $k$ are reported in brackets.}
\begin{tabular}{l|cc|cc}
\toprule
\multirow{2}{*}{Method} & \multicolumn{2}{c|}{Small Scale} & \multicolumn{2}{c}{Large Scale} \\
  & RT & $\Delta$RO & RT & $\Delta$RO \\
\midrule
Random & 0.0 (0.0) & 62.9 (16.4) & 0.0 (0.0) & 20.3 (2.4) \\
FasterPAM & 100.0 (10.6) & 0.0 (0.3) & NaN & NaN \\
Alternate & 161.1 (43.2) & 20.0 (7.4) & NaN & NaN \\
FasterCLARA-5 & 2.8 (0.4) & 13.0 (1.5) & 15.0 (0.5) & 8.0 (0.6) \\
FasterCLARA-50 & 30.0 (1.3) & 10.9 (0.8) & 161.7 (3.2) & 7.1 (0.4) \\
kmc2-20 & 14.5 (0.5) & 31.3 (4.4) & 0.5 (0.0) & 18.2 (2.4) \\
kmc2-100 & 72.2 (1.0) & 31.9 (4.9) & 2.4 (0.2) & 17.6 (2.3) \\
kmc2-200 & 153.6 (9.2) & 33.0 (6.1) & 5.2 (0.3) & 18.6 (2.6) \\
k-means++ & 1.6 (0.1) & 30.4 (4.8) & 78.8 (4.1) & 18.4 (2.7) \\
LS-k-means++-5 & 37.2 (0.5) & 23.5 (3.3) & 97.1 (2.1) & 15.3 (1.8) \\
LS-k-means++-10 & 73.1 (2.2) & 20.1 (2.9) & 121.6 (2.9) & 13.7 (1.7) \\
BanditPAM++-0 & 930.2 (40.3) & 3.6 (0.3) & NaN & NaN \\
BanditPAM++-2 & 1670.1 (41.3) & 2.8 (0.3) & NaN & NaN \\
BanditPAM++-5 & 2880.7 (65.5) & 2.2 (0.2) & NaN & NaN \\
OneBatch-lwcs & 15.1 (1.2) & 12.3 (1.5) & 118.2 (7.9) & 2.7 (0.6) \\
OneBatch-unif & 15.1 (1.8) & 3.9 (0.7) & 104.2 (8.8) & 1.2 (0.4) \\
OneBatch-debias & 15.7 (3.0) & 3.7 (0.7) & 100.0 (6.3) & 0.8 (0.3) \\
OneBatch-nniw & 15.5 (1.6) & 1.7 (0.5) & 100.0 (4.1) & 0.0 (0.3) \\
\bottomrule
\end{tabular}
\label{avg-results-table}
\end{table}

\appsubsectiontitle{Detailed Results Small Scale}

\begin{table}[H]
\centering
\caption{\textbf{Relative Time (RT) per dataset for the ``small scale'' experiments}. The scores are averaged over the five repetitions of the experiment and the three values of $k \in [10, 50, 100]$. RT is given in percentage. The standard deviations are reported in brackets.}
\begin{tabular}{l|ccccc}
\toprule
\multirow{2}{*}{\backslashbox{Methods}{Datasets}} & abalone & bankruptcy & drybean & letter & mapping \\
& & & & & \\
\midrule
Random & 0.0 (0.0) & 0.0 (0.0) & 0.0 (0.0) & 0.0 (0.0) & 0.0 (0.0) \\
FasterPAM & 100.0 (5.6) & 100.0 (19.7) & 100.0 (8.4) & 100.0 (10.1) & 100.0 (9.3) \\
Alternate & 150.9 (34.7) & 97.9 (23.8) & 321.2 (124.6) & 140.3 (23.7) & 95.0 (9.4) \\
FasterCLARA-5 & 6.6 (1.2) & 1.9 (0.1) & 1.9 (0.5) & 1.1 (0.1) & 2.4 (0.1) \\
FasterCLARA-50 & 72.6 (1.7) & 21.3 (1.8) & 19.1 (1.3) & 11.3 (0.2) & 25.8 (1.6) \\
kmc2-20 & 59.4 (1.9) & 2.0 (0.0) & 4.6 (0.3) & 1.9 (0.0) & 4.5 (0.1) \\
kmc2-100 & 296.6 (3.5) & 9.8 (0.1) & 21.7 (0.3) & 9.6 (0.1) & 23.5 (1.1) \\
kmc2-200 & 634.0 (36.9) & 20.4 (0.9) & 45.9 (3.9) & 19.8 (0.9) & 47.9 (3.2) \\
k-means++ & 2.1 (0.1) & 2.2 (0.1) & 1.3 (0.0) & 0.9 (0.0) & 1.6 (0.1) \\
LS-k-means++-5 & 103.3 (1.6) & 8.3 (0.1) & 31.8 (0.3) & 20.6 (0.3) & 21.9 (0.1) \\
LS-k-means++-10 & 202.5 (4.3) & 14.3 (0.2) & 63.7 (1.1) & 40.3 (0.2) & 44.7 (5.4) \\
BanditPAM++-0 & 1388.5 (87.7) & 722.7 (41.8) & 858.0 (8.8) & 733.2 (7.7) & 948.8 (55.7) \\
BanditPAM++-2 & 2729.9 (52.8) & 1073.0 (43.8) & 1491.7 (30.3) & 1270.7 (7.8) & 1785.0 (71.9) \\
BanditPAM++-5 & 5394.9 (222.2) & 1574.8 (28.8) & 2434.9 (33.2) & 2068.0 (5.9) & 2930.8 (37.3) \\
OneBatch-lwcs & 34.3 (3.6) & 7.5 (0.6) & 12.2 (0.5) & 7.8 (0.5) & 13.6 (1.0) \\
OneBatch-unif & 31.3 (2.5) & 6.8 (0.2) & 13.3 (2.9) & 8.8 (1.9) & 15.4 (1.5) \\
OneBatch-debias & 36.8 (9.4) & 7.1 (0.2) & 11.9 (0.5) & 8.0 (1.2) & 14.7 (3.6) \\
OneBatch-nniw & 34.0 (3.7) & 7.1 (0.4) & 14.3 (3.0) & 8.5 (0.5) & 13.5 (0.7) \\
\bottomrule
\end{tabular}
\end{table}

\begin{table}[H]
\centering
\caption{\textbf{Delta Relative Objective ($\Delta$RO) per dataset for the ``small scale'' experiments}. The scores are averaged over the five repetitions of the experiment and the three values of $k \in [10, 50, 100]$. $\Delta$RO is given in percentage. The standard deviations are reported in brackets.}
\begin{tabular}{l|ccccc}
\toprule
\multirow{2}{*}{\backslashbox{Methods}{Datasets}} & abalone & bankruptcy & drybean & letter & mapping \\
& & & & & \\
\midrule
Random & 82.5 (15.1) & 32.4 (3.0) & 154.6 (59.7) & 26.4 (2.1) & 18.6 (2.3) \\
FasterPAM & 0.0 (0.2) & 0.0 (0.1) & 0.0 (0.9) & 0.0 (0.1) & 0.0 (0.1) \\
Alternate & 30.1 (10.9) & 12.1 (3.8) & 41.9 (19.8) & 9.4 (1.2) & 6.2 (1.3) \\
FasterCLARA-5 & 13.5 (1.4) & 12.2 (0.7) & 16.3 (3.8) & 13.5 (0.9) & 9.5 (0.6) \\
FasterCLARA-50 & 10.8 (0.8) & 11.2 (0.6) & 12.1 (1.7) & 11.9 (0.5) & 8.4 (0.4) \\
kmc2-20 & 40.3 (5.7) & 30.5 (3.9) & 42.4 (8.4) & 24.9 (2.1) & 18.4 (1.7) \\
kmc2-100 & 41.2 (6.6) & 31.3 (4.4) & 41.0 (9.4) & 26.0 (2.0) & 19.8 (2.4) \\
kmc2-200 & 45.7 (10.2) & 30.3 (4.7) & 42.5 (11.0) & 27.2 (1.7) & 19.5 (2.7) \\
k-means++ & 39.5 (7.7) & 31.7 (3.1) & 35.0 (7.8) & 26.1 (4.1) & 19.4 (1.6) \\
LS-k-means++-5 & 30.9 (6.0) & 23.7 (3.1) & 22.5 (3.0) & 22.6 (2.3) & 17.5 (2.1) \\
LS-k-means++-10 & 25.6 (4.7) & 21.1 (3.1) & 18.3 (3.2) & 20.5 (1.9) & 14.9 (1.4) \\
BanditPAM++-0 & 5.2 (0.6) & 3.6 (0.3) & 4.7 (0.5) & 2.1 (0.1) & 2.3 (0.2) \\
BanditPAM++-2 & 3.7 (0.4) & 2.6 (0.2) & 4.1 (0.5) & 1.8 (0.1) & 1.8 (0.2) \\
BanditPAM++-5 & 2.9 (0.4) & 2.1 (0.2) & 3.1 (0.3) & 1.6 (0.0) & 1.3 (0.1) \\
OneBatch-lwcs & 10.6 (1.2) & 3.1 (0.5) & 41.6 (5.0) & 4.1 (0.5) & 2.3 (0.4) \\
OneBatch-unif & 3.5 (0.6) & 3.6 (0.4) & 6.8 (1.7) & 3.3 (0.6) & 2.6 (0.3) \\
OneBatch-debias & 3.1 (0.6) & 3.0 (0.4) & 6.7 (1.7) & 3.3 (0.6) & 2.3 (0.3) \\
OneBatch-nniw & 1.4 (0.4) & 1.6 (0.2) & 2.4 (1.2) & 1.8 (0.2) & 1.4 (0.3) \\
\bottomrule
\end{tabular}
\end{table}

\begin{figure}[H]
    \centering
    \includegraphics[width=0.95\linewidth]{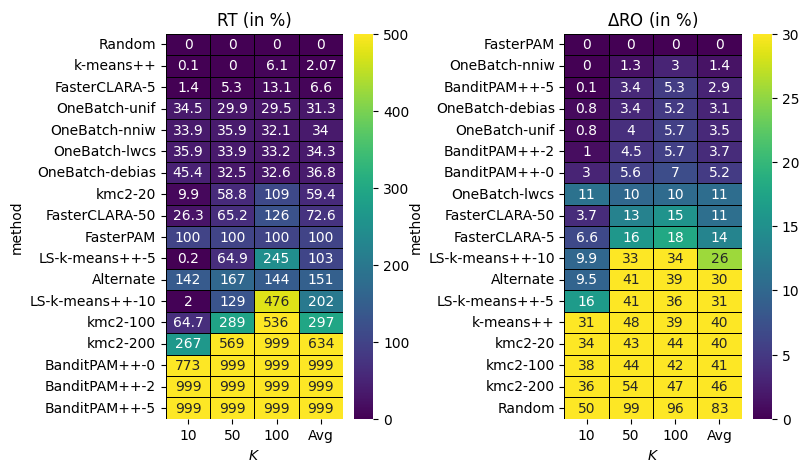} 
    \caption{\textbf{RT and $\Delta$RO for Abalone}}
\end{figure}

\begin{figure}[H]
    \centering
    \includegraphics[width=0.95\linewidth]{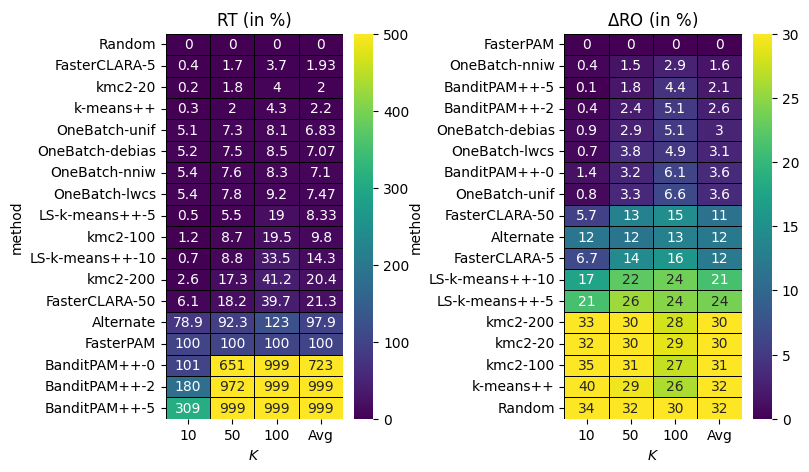} 
    \caption{\textbf{RT and $\Delta$RO for Bankruptcy}}
    
\end{figure}

\begin{figure}[H]
    \centering
    \includegraphics[width=0.95\linewidth]{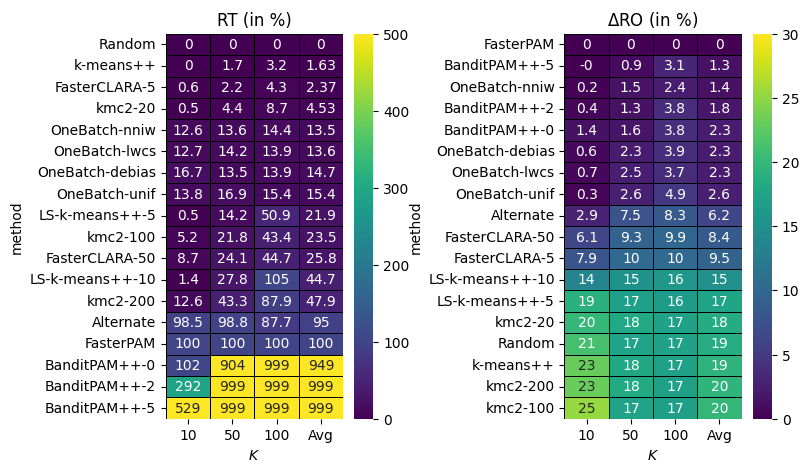} 
    \caption{\textbf{RT and $\Delta$RO for Mapping}}
    
\end{figure}

\begin{figure}[H]
    \centering
    \includegraphics[width=0.95\linewidth]{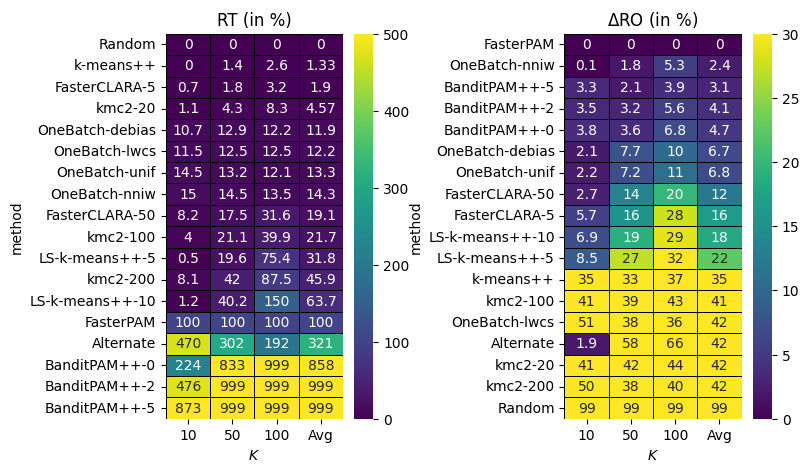} 
    \caption{\textbf{RT and $\Delta$RO for Drybean}}
    
\end{figure}

\begin{figure}[H]
    \centering
    \includegraphics[width=0.95\linewidth]{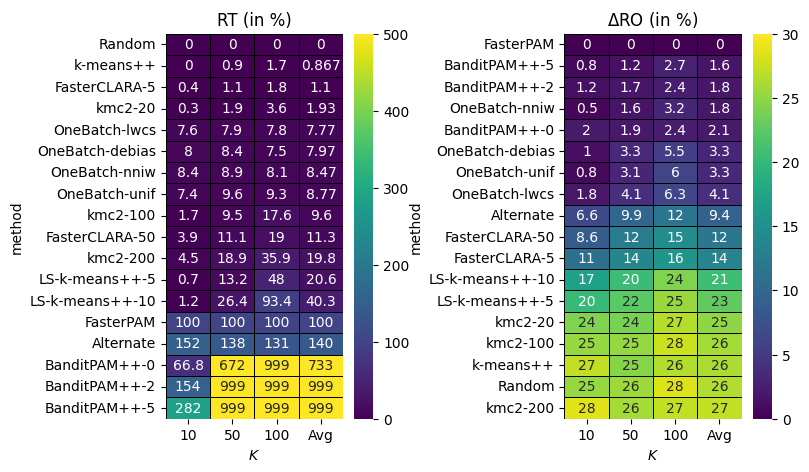} 
    \caption{\textbf{RT and $\Delta$RO for Letter}}
    
\end{figure}

\clearpage

\appsubsectiontitle{Detailed Results Large Scale}

\begin{table}[H]
\centering
\caption{\textbf{Relative Time (RT) per dataset for the ``large scale'' experiments}. The scores are averaged over the five repetitions of the experiment and the three values of $k \in [10, 50, 100]$. RT is given in percentage. The standard deviations are reported in brackets.}
\begin{tabular}{l|ccccc}
\toprule
\multirow{2}{*}{\backslashbox{Methods}{Datasets}} & cifar & covertype & dota2 & mnist & monitor-gas  \\
& & & & & \\
\midrule
Random & 0.0 (0.0) & 0.0 (0.0) & 0.0 (0.0) & 0.0 (0.0) & 0.0 (0.0) \\
FasterCLARA-5 & 19.8 (0.4) & 14.3 (0.0) & 12.4 (0.3) & 15.1 (1.0) & 13.4 (0.8) \\
FasterCLARA-50 & 193.7 (1.3) & 169.6 (0.2) & 144.8 (2.1) & 159.4 (3.6) & 140.9 (8.9) \\
kmc2-20 & 0.4 (0.0) & 0.1 (0.0) & 1.0 (0.0) & 0.4 (0.0) & 0.5 (0.0) \\
kmc2-100 & 2.1 (0.2) & 0.5 (0.0) & 5.0 (0.3) & 1.7 (0.3) & 2.4 (0.2) \\
kmc2-200 & 4.2 (0.2) & 1.4 (0.0) & 11.3 (0.9) & 3.3 (0.1) & 5.8 (0.3) \\
k-means++ & 21.0 (0.3) & 76.0 (2.3) & 17.8 (0.2) & 17.2 (1.3) & 261.9 (16.2) \\
LS-k-means++-5 & 24.8 (0.6) & 94.8 (4.1) & 53.7 (0.5) & 26.0 (1.2) & 286.3 (3.9) \\
LS-k-means++-10 & 29.1 (1.6) & 105.8 (5.8) & 88.9 (0.2) & 35.3 (1.5) & 348.7 (5.4) \\
OneBatch-lwcs & 100.4 (0.4) & 135.4 (9.7) & 101.6 (2.0) & 118.4 (22.7) & 135.0 (5.0) \\
OneBatch-unif & 116.9 (19.0) & 117.5 (12.6) & 95.2 (7.6) & 92.1 (2.2) & 99.1 (2.5) \\
OneBatch-debias & 109.9 (6.4) & 114.1 (12.8) & 81.0 (3.8) & 92.8 (2.1) & 102.3 (6.6) \\
OneBatch-nniw & 100.0 (0.3) & 100.0 (2.9) & 100.0 (6.1) & 100.0 (7.5) & 100.0 (3.8) \\
\bottomrule
\end{tabular}
\label{small-scale-rt-table}
\end{table}

\begin{table}[H]
\centering
\caption{\textbf{Delta Relative Objective ($\Delta$RO) per dataset for the ``large scale'' experiments}. The scores are averaged over the five repetitions of the experiment and the three values of $k \in [10, 50, 100]$. $\Delta$RO is given in percentage. The standard deviations are reported in brackets.}
\begin{tabular}{l|ccccc}
\toprule
\multirow{2}{*}{\backslashbox{Methods}{Datasets}} & cifar & covertype & dota2 & mnist & monitor-gas \\
& & & & & \\
\midrule
Random & 16.8 (1.4) & 24.9 (3.4) & 12.2 (1.9) & 15.5 (1.6) & 31.9 (3.5) \\
FasterCLARA-5 & 7.9 (0.6) & 9.7 (0.8) & 3.8 (0.3) & 8.1 (0.5) & 10.7 (1.0) \\
FasterCLARA-50 & 7.3 (0.5) & 8.5 (0.8) & 3.2 (0.1) & 7.1 (0.3) & 9.2 (0.6) \\
kmc2-20 & 18.1 (2.8) & 22.2 (3.5) & 8.9 (1.6) & 16.0 (1.8) & 25.8 (2.1) \\
kmc2-100 & 17.2 (1.5) & 22.0 (2.3) & 8.8 (2.2) & 15.5 (2.1) & 24.2 (3.1) \\
kmc2-200 & 19.2 (3.2) & 22.5 (3.2) & 8.6 (1.9) & 16.4 (2.2) & 26.1 (2.7) \\
k-means++ & 19.0 (2.5) & 21.7 (2.9) & 8.6 (2.4) & 16.3 (1.4) & 26.2 (4.4) \\
LS-k-means++-5 & 17.2 (2.4) & 17.7 (2.0) & 6.2 (0.9) & 13.7 (1.4) & 21.8 (2.5) \\
LS-k-means++-10 & 16.2 (1.9) & 15.8 (1.7) & 5.6 (0.6) & 12.6 (1.8) & 18.1 (2.7) \\
OneBatch-lwcs & -0.3 (0.2) & 3.8 (1.1) & 0.7 (0.3) & 0.8 (0.3) & 8.6 (1.3) \\
OneBatch-unif & 0.3 (0.2) & 1.7 (0.5) & 0.6 (0.2) & 0.9 (0.2) & 2.4 (0.9) \\
OneBatch-debias & -0.7 (0.1) & 1.9 (0.5) & 0.2 (0.1) & 0.6 (0.2) & 2.1 (0.6) \\
OneBatch-nniw & 0.0 (0.2) & 0.0 (0.2) & 0.0 (0.1) & 0.0 (0.2) & 0.0 (0.6) \\
\bottomrule
\end{tabular}
\end{table}

\begin{figure}[H]
    \centering
    \includegraphics[width=0.95\linewidth]{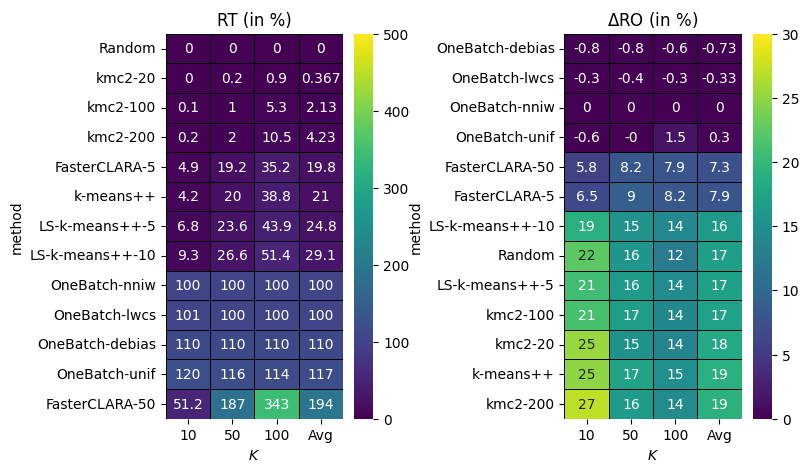} 
    \caption{\textbf{RT and $\Delta$RO for CIFAR}}
    
\end{figure}

\begin{figure}[H]
    \centering
    \includegraphics[width=0.95\linewidth]{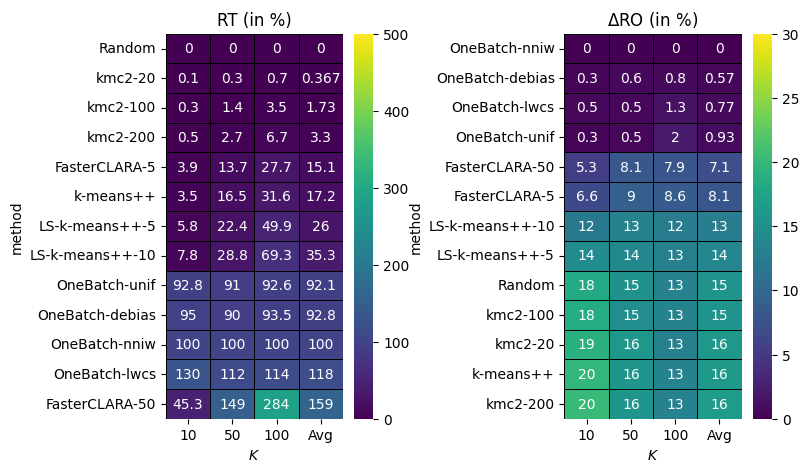} 
    \caption{\textbf{RT and $\Delta$RO for MNIST}}
    
\end{figure}

\begin{figure}[H]
    \centering
    \includegraphics[width=0.95\linewidth]{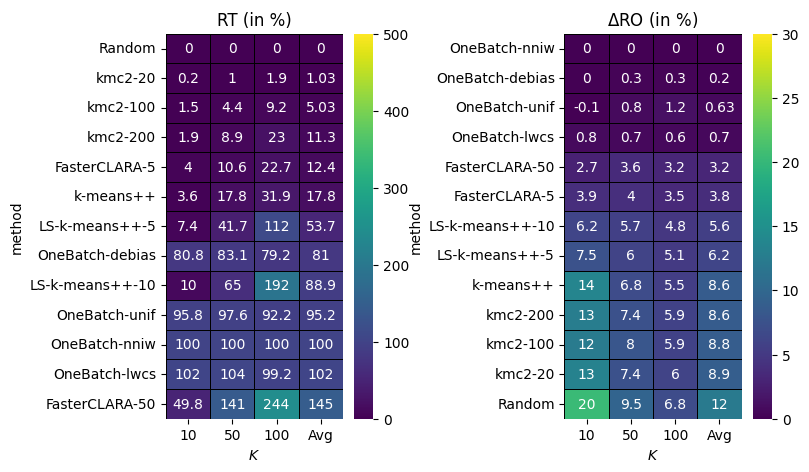} 
    \caption{\textbf{RT and $\Delta$RO for Dota2}}
    
\end{figure}

\begin{figure}[H]
    \centering
    \includegraphics[width=0.95\linewidth]{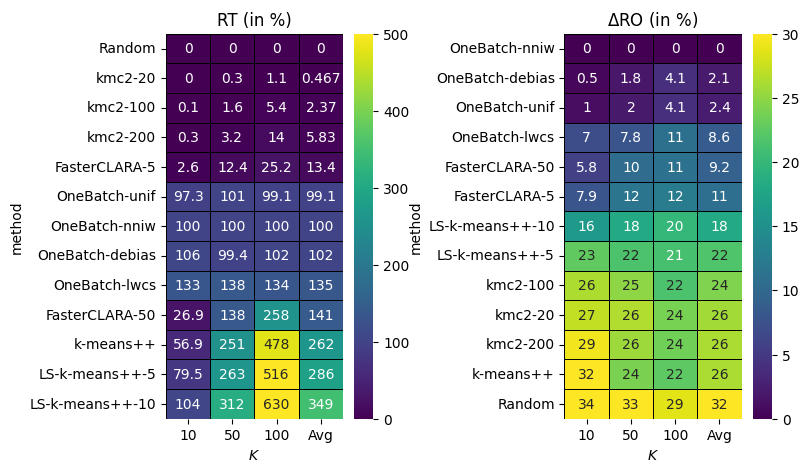} 
    \caption{\textbf{RT and $\Delta$RO for Monitor-gas}}
    
\end{figure}

\begin{figure}[H]
    \centering
    \includegraphics[width=0.95\linewidth]{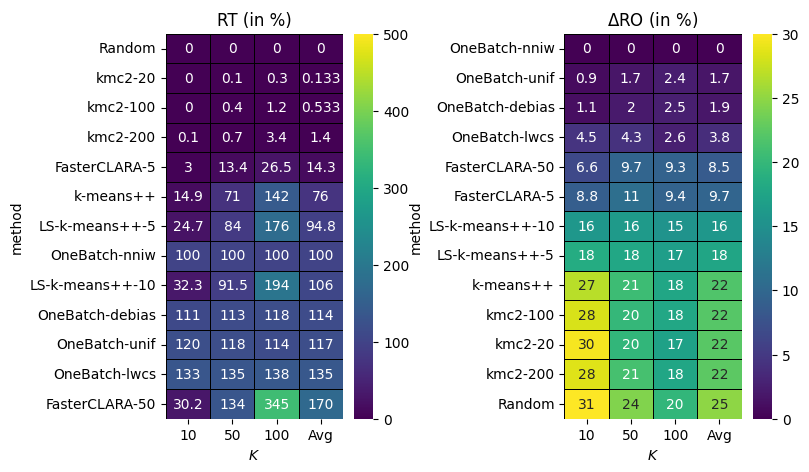} 
    \caption{\textbf{RT and $\Delta$RO for Covertype}}
    
\end{figure}

\clearpage

\appsectiontitle{Pareto Front}

This section presents the Pareto front (in red) for Objective vs Time graphs for each dataset and the two configurations $k=10$ and $k=100$. Algorithms belonging to the Pareto front are ``optimal'' for at least one objective/time trade-off. In contrast, the algorithms out of the Pareto front are ``suboptimal'' because another algorithm provides a better objective with less running time.

We observe that, for the small-scale datasets, $k$-means++, FasterCLARA-5, OneBatch-nniw and FasterPAM belong to the Pareto fronts. The Pareto fronts for the large-scale datasets include kmc2-20, FasterCLARA-5 and OneBatch-nniw.

\begin{figure}[!htb]
    \centering
        \centering
        \includegraphics[width=0.7\linewidth]{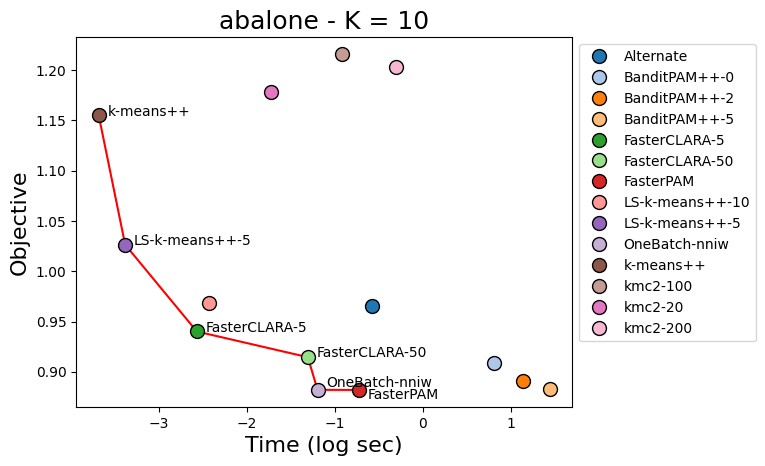} 
    \caption{\textbf{Objective vs Time: Abalone ($k=10$)}}
    
\end{figure}

\begin{figure}[!htb]
    \centering
        \centering
        \includegraphics[width=0.7\linewidth]{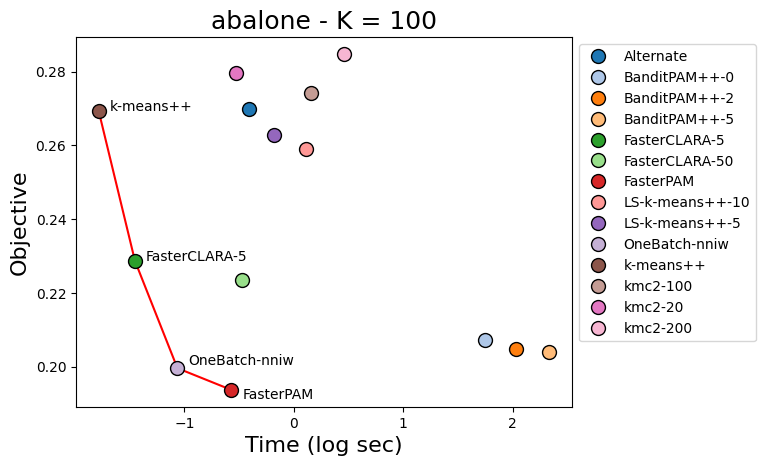} 
    \caption{\textbf{Objective vs Time: Abalone ($k=100$)}}
    
\end{figure}

\begin{figure}[!htb]
    \centering
        \centering
        \includegraphics[width=0.8\linewidth]{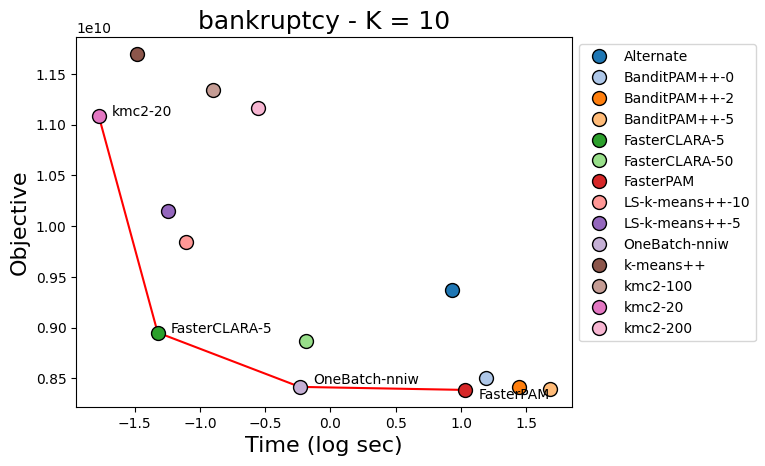} 
    \caption{\textbf{Objective vs Time: Bankruptcy ($k=10$)}}
    
\end{figure}

\begin{figure}[!htb]
    \centering
        \centering
        \includegraphics[width=0.8\linewidth]{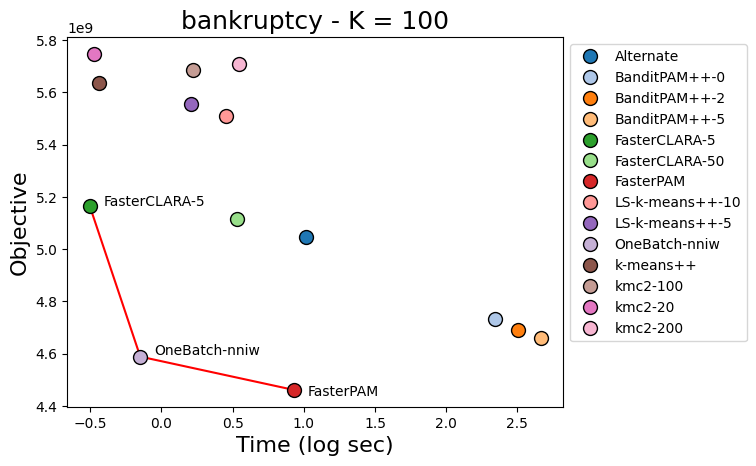} 
    \caption{\textbf{Objective vs Time: Bankruptcy ($k=100$)}}
    
\end{figure}

\begin{figure}[!htb]
    \centering
        \centering
        \includegraphics[width=0.8\linewidth]{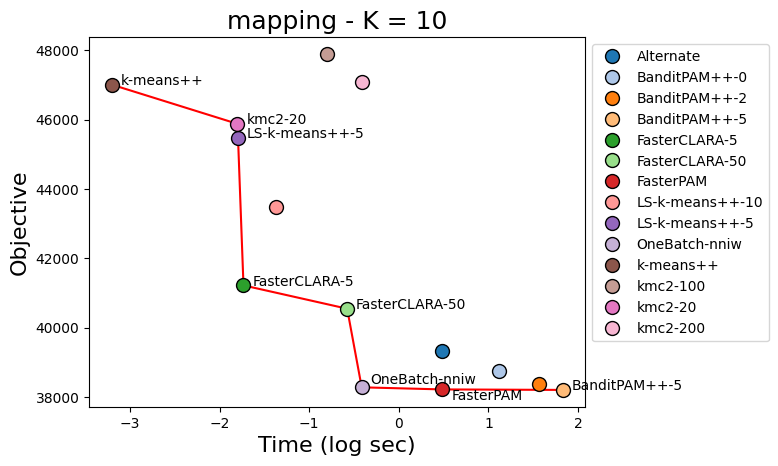} 
    \caption{\textbf{Objective vs Time: Mapping ($k=10$)}}
    
\end{figure}

\begin{figure}[!htb]
    \centering
        \centering
        \includegraphics[width=0.8\linewidth]{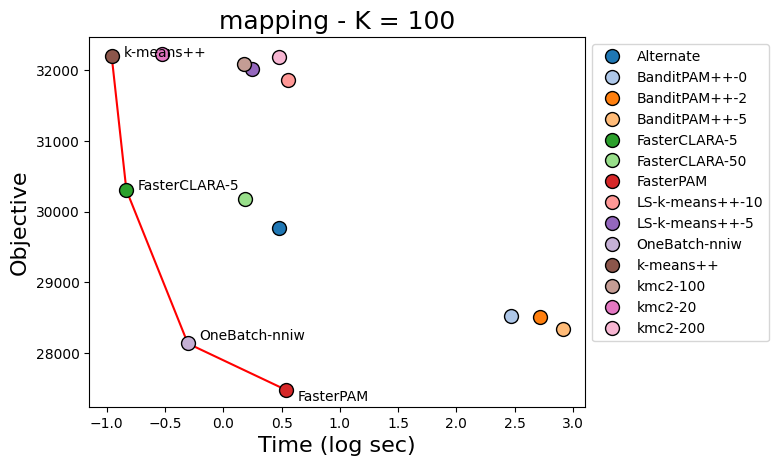} 
    \caption{\textbf{Objective vs Time: Mapping ($k=100$)}}
    
\end{figure}

\begin{figure}[!htb]
    \centering
        \centering
        \includegraphics[width=0.8\linewidth]{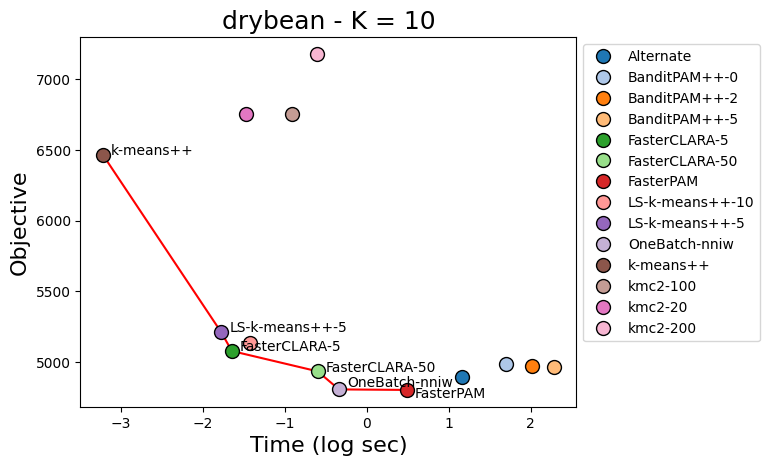} 
    \caption{\textbf{Objective vs Time: Drybean ($k=10$)}}
    
\end{figure}

\begin{figure}[!htb]
    \centering
        \centering
        \includegraphics[width=0.8\linewidth]{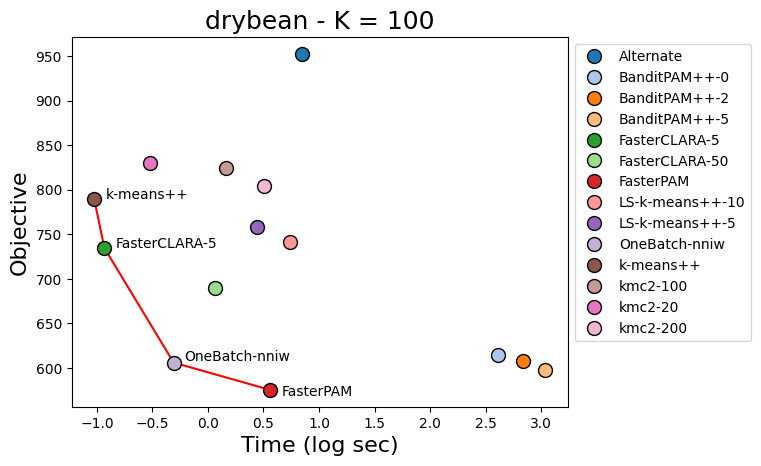} 
    \caption{\textbf{Objective vs Time: Drybean ($k=100$)}}
    
\end{figure}

\begin{figure}[!htb]
    \centering
        \centering
        \includegraphics[width=0.8\linewidth]{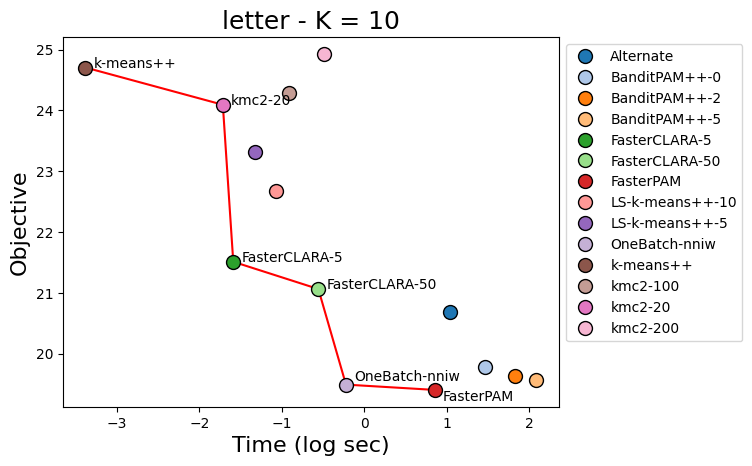} 
    \caption{\textbf{Objective vs Time: Letter ($k=10$)}}
    
\end{figure}

\begin{figure}[!htb]
    \centering
        \centering
        \includegraphics[width=0.8\linewidth]{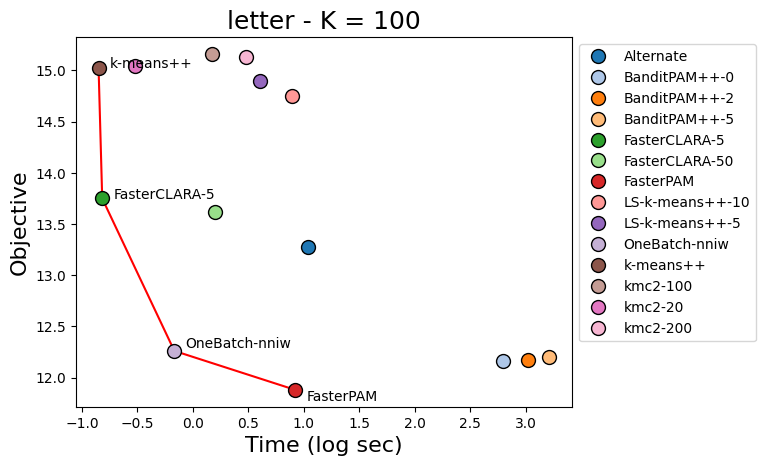} 
    \caption{\textbf{Objective vs Time: Letter ($k=100$)}}
    
\end{figure}

\begin{figure}[!htb]
    \centering
        \centering
        \includegraphics[width=0.8\linewidth]{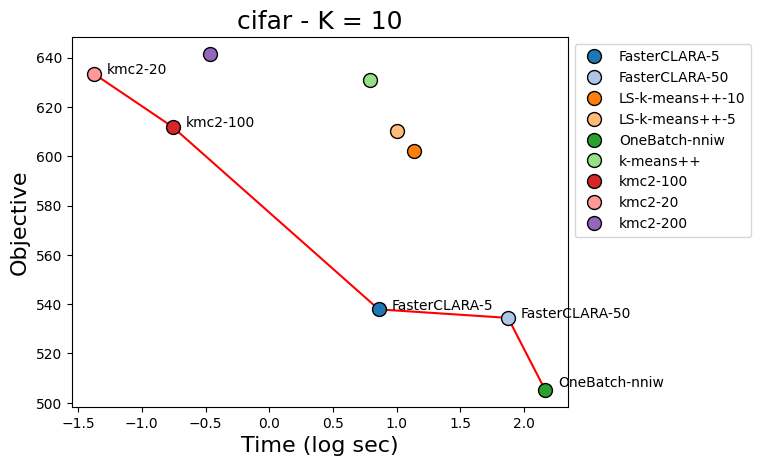} 
    \caption{\textbf{Objective vs Time: CIFAR ($k=10$)}}
    
\end{figure}

\begin{figure}[!htb]
    \centering
        \centering
        \includegraphics[width=0.8\linewidth]{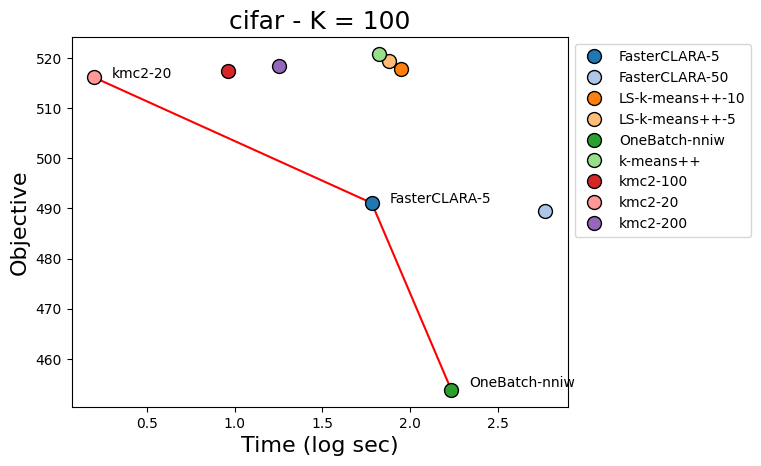} 
    \caption{\textbf{Objective vs Time: CIFAR ($k=100$)}}
    
\end{figure}

\begin{figure}[!htb]
    \centering
        \centering
        \includegraphics[width=0.8\linewidth]{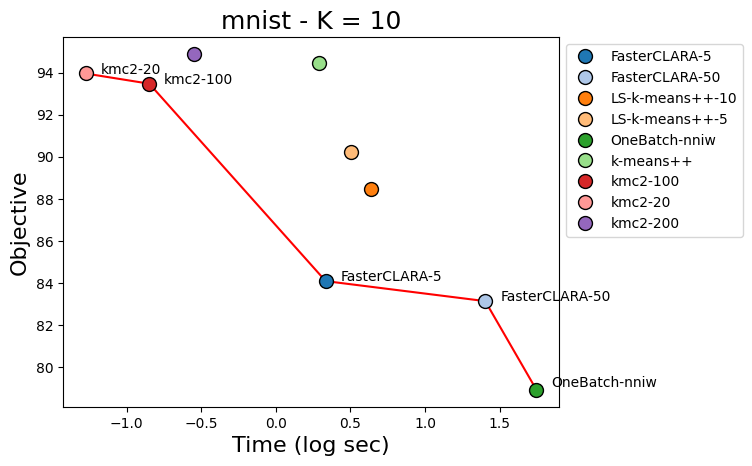} 
    \caption{\textbf{Objective vs Time: MNIST ($k=10$)}}
    
\end{figure}

\begin{figure}[!htb]
    \centering
        \centering
        \includegraphics[width=0.8\linewidth]{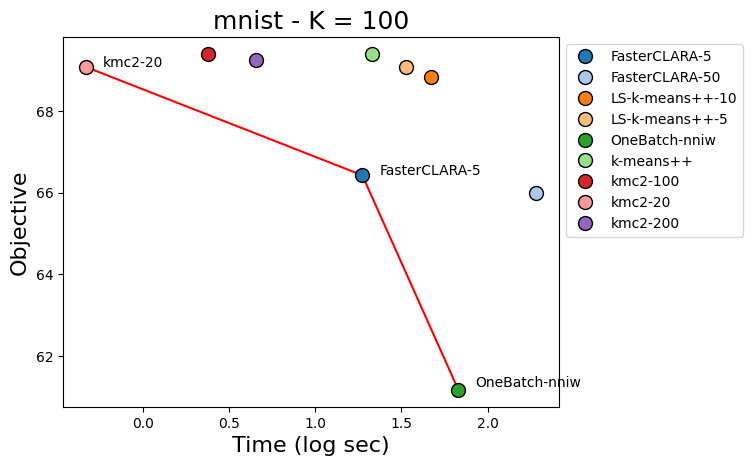} 
    \caption{\textbf{Objective vs Time: MNIST ($k=100$)}}
    
\end{figure}

\begin{figure}[!htb]
    \centering
        \centering
        \includegraphics[width=0.8\linewidth]{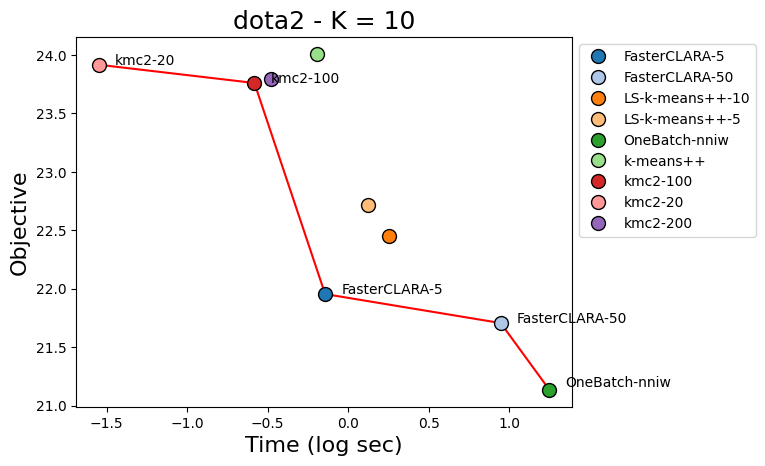} 
    \caption{\textbf{Objective vs Time: Dota2 ($k=10$)}}
    
\end{figure}

\begin{figure}[!htb]
    \centering
        \centering
        \includegraphics[width=0.8\linewidth]{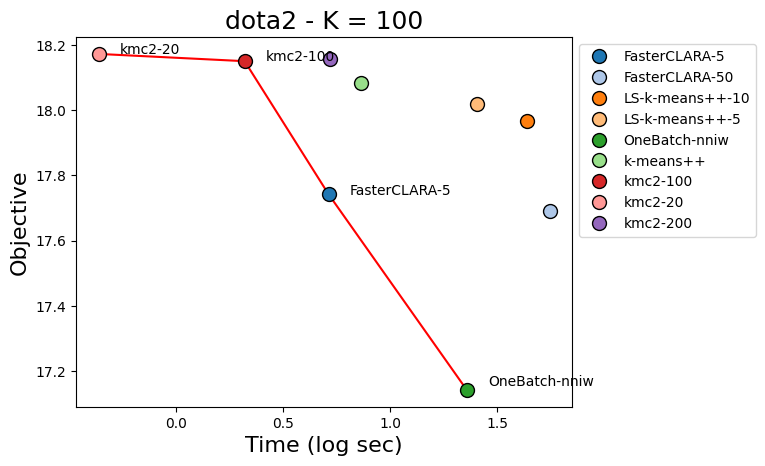} 
    \caption{\textbf{Objective vs Time: Dota2 ($k=100$)}}
    
\end{figure}

\begin{figure}[!htb]
    \centering
        \centering
        \includegraphics[width=0.8\linewidth]{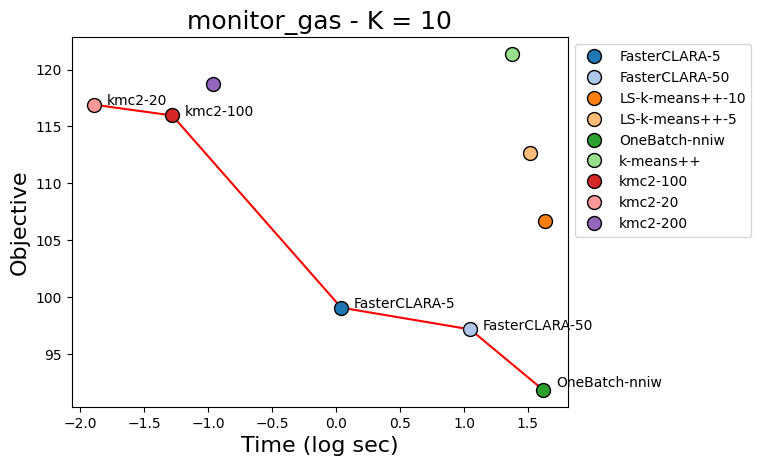} 
    \caption{\textbf{Objective vs Time: Monitor-gas ($k=10$)}}
    
\end{figure}

\begin{figure}[!htb]
    \centering
        \centering
        \includegraphics[width=0.8\linewidth]{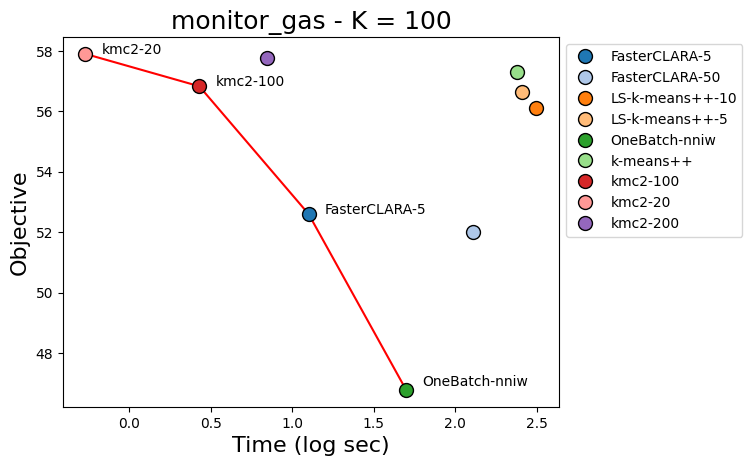} 
    \caption{\textbf{Objective vs Time: Monitor-gas ($k=100$)}}
    
\end{figure}

\begin{figure}[!htb]
    \centering
        \centering
        \includegraphics[width=0.8\linewidth]{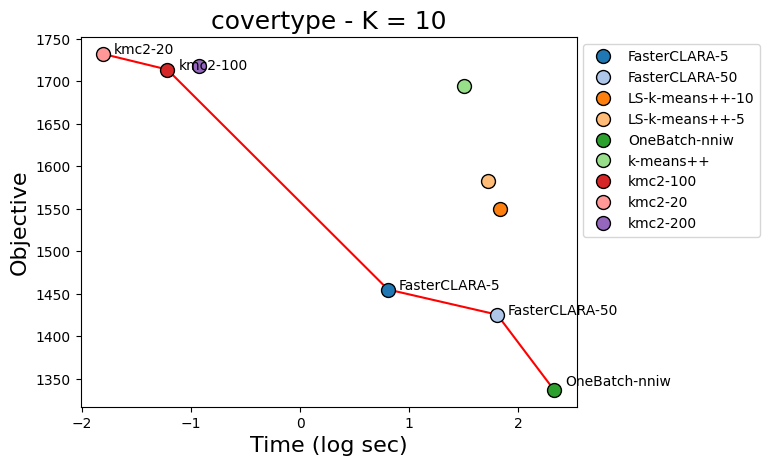} 
    \caption{\textbf{Objective vs Time: Covertype ($k=10$)}}
    
\end{figure}

\begin{figure}[!htb]
    \centering
        \centering
        \includegraphics[width=0.8\linewidth]{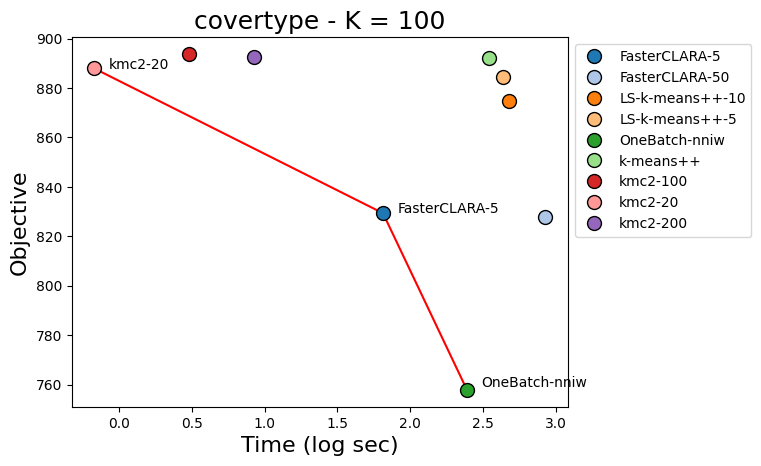} 
    \caption{\textbf{Objective vs Time: Covertype ($k=100$)}}
    
\end{figure}

\end{appendices}

\end{document}